\documentclass[10pt,twocolumn,a4paper]{article}

\usepackage[T1]{fontenc}
\usepackage[utf8]{inputenc}
\usepackage{palatino}
\usepackage{mathpazo}
\usepackage{microtype}
\usepackage{geometry}
\usepackage{amsmath,amssymb,amsthm}
\usepackage{booktabs}
\usepackage{fancyhdr}
\usepackage[table]{xcolor}
\usepackage{parskip}
\usepackage{enumitem}
\usepackage{titlesec}
\usepackage{graphicx}
\usepackage{float}
\usepackage{tikz}
\usetikzlibrary{arrows.meta,positioning,fit,backgrounds,calc}

\usepackage{hyperref}
\hypersetup{
  hidelinks,
  pdftitle={Runtime-Certified Bounded-Error Quantized Attention},
  pdfauthor={Dean Calver},
  pdfsubject={KV cache quantization with formal error bounds and FP16 fallback},
  pdfkeywords={KV cache, quantization, attention, LLM inference, error bounds}
}

\newtheorem{theorem}{Theorem}[section]
\newtheorem{corollary}[theorem]{Corollary}
\newtheorem{lemma}[theorem]{Lemma}
\theoremstyle{definition}

\theoremstyle{remark}

\titlespacing*{\section}{0pt}{12pt}{6pt}
\titlespacing*{\subsection}{0pt}{8pt}{3pt}
\newcommand{\R}{\mathbb{R}}
\newcommand{\norm}[1]{\lVert #1 \rVert}
\newcommand{\torchsdpa}{\texttt{torch.scaled\_\hspace{0pt}dot\_\hspace{0pt}product\_\hspace{0pt}attention}}

\begin{document}

\twocolumn[{%
\begin{center}
  {\LARGE\bfseries Runtime-Certified Bounded-Error}\\[4pt]
  {\LARGE\bfseries Quantized Attention}\\[10pt]
  {\large\itshape Per-Head, Per-Step Error Bounds for Compressed KV Caches\\
  with Dense FP16 Fallback}\\[8pt]
  Dean Calver\\[2pt]
  {\small Independent Researcher}\\[2pt]
  {\small\ttfamily deano@cloudpixies.com}
\end{center}

\vspace{4pt}
\hrule
\vspace{8pt}


\begin{abstract}
KV cache quantization reduces the memory cost of long-context LLM inference, but introduces approximation error that is typically validated only empirically. Existing systems rely on average-case robustness, with no mechanism to detect or recover from failures at runtime.

We present a tiered KV cache architecture that enables runtime-certified attention: INT8 keys and INT4 values are stored in GPU memory, while FP16 originals are retained in system RAM for deterministic fallback. A two-term error decomposition yields per-head, per-step bounds on (i) attention distribution distortion from key quantization and (ii) value reconstruction error. These bounds are computed online and used to drive adaptive precision selection and a multi-stage fallback ladder, which guarantees recovery to the exact dense attention output when required.

Across PG-19, NIAH, and RULER benchmarks on LLaMA~3.1-8B with contexts up to 128K, the system matches dense FP16 KV quality within noise for language modelling and retrieval tasks, while recovering catastrophic failures observed in naive INT8/INT4 baselines. Value-sensitive tasks at short context expose a controlled trade-off between compression and fidelity, which can be eliminated via tighter value tolerances or FP16-value fallback.

The certification is local (per-head, per-step) and does not guarantee end-to-end model correctness, but ensures that each attention computation is either bounded relative to an FP16 reference or exactly recovered via fallback. This reframes KV cache quantization as a runtime-verified computation rather than a fixed approximation. The goal is not raw speedups, but enabling safe deployment of aggressive KV compression under strict quality constraints.
\end{abstract}

\vspace{4pt}
\hrule
\vspace{8pt}
}] 

\section{Introduction}
\label{sec:intro}

Autoregressive LLM decoding at long context lengths is dominated by the cost of reading the key-value (KV) cache from GPU memory.  Each decode step fetches cached keys and values for every attention head, a cost linear in sequence length.  At 128K context on LLaMA~3.1-8B (32 layers, 32 query heads, 8 KV heads under GQA, and head dimension 128), the FP16 KV cache occupies roughly 17\,GB and consumes the majority of available memory bandwidth during decode.

KV cache quantization~\cite{kvquant,kivi,qjl,innerq} addresses this by storing keys and values in reduced-precision formats---typically INT8, INT4, or lower.  The savings are substantial: an INT8 cache halves memory usage, and INT4 quarters it.  But quantization introduces error, and the relationship between quantization parameters and output quality is complex.  The same quantization scheme that works well on one prompt may fail on another, depending on the attention pattern, the key/value distributions, and the specific tokens being attended to.

The standard response is empirical validation: run a benchmark suite, check that perplexity doesn't degrade beyond a threshold, and deploy.  This works in practice but provides no per-step guarantee.  A system that achieves $\Delta\text{ppl} < 0.1$ on average may still exhibit large step-wise deviations on individual decode steps---and there is no mechanism to detect when this happens, let alone recover from it.

KV cache quantization and compressed-domain attention are well-established techniques.  What is missing is a \emph{guarantee}: a way to bound, at each decode step, how far the quantized output deviates from an unquantized reference---and a way to recover the production dense output when the bound is too large.\footnote{The paper distinguishes three reference outputs: $O_{\mathrm{dense}}$ (production Dense baseline via \torchsdpa{}), $O_{\mathrm{ref}}$ (same certified kernel with unquantized FP16 KV), and $O_{\mathrm{quant}}$ (certified kernel with quantized KV).  The formal bounds (Section~\ref{sec:bounds}) certify $\norm{O_{\mathrm{quant}} - O_{\mathrm{ref}}}$; the arithmetic-path gap $\norm{O_{\mathrm{ref}} - O_{\mathrm{dense}}}$ is characterised in Section~\ref{sec:numerical}.  On fallback, the system returns $O_{\mathrm{dense}}$ directly.}

This paper develops such a guarantee together with a practical recovery mechanism.  In this sense, it reframes KV cache quantization as a runtime-verified computation rather than a fixed approximation.  Unlike prior adaptive-precision systems (Section~\ref{sec:related}), it provides a runtime-computable error bound and a deterministic recovery path to the exact dense baseline.  The key ideas are:

\textbf{1.\ Tiered storage with dense fallback.}  Quantized keys and values live in GPU memory; full-precision originals live in pinned system RAM\@.  The FP16 data is not discarded after quantization---it is retained as a safety net.  Any block can be escalated to full precision at any time, at the cost of a host-to-device page-in from system RAM (overlappable with GPU compute).  The terminal fallback path uses the same \torchsdpa{} code path and FP16 KV inputs as the Dense baseline.  This fallback is unconditional in correctness, though not in cost: it requires transient staging memory for one layer's KV and host-to-device bandwidth for the page-in, but never requires the full dense cache to be persistently resident in VRAM (Section~\ref{sec:fallback}).

\textbf{2.\ Formal error bounds with runtime monitoring.}  A two-term error decomposition bounds the output perturbation from key compression and value compression independently.  The key compression error depends on how quantization shifts the softmax distribution; the value compression error depends on the per-token reconstruction quality.  Both bounds are computable at runtime from quantities the system already tracks (quantization scales, query norms, value ranges), enabling per-head, per-step precision decisions.

\textbf{3.\ Adaptive per-block precision selection.}  Rather than applying a single quantization policy uniformly, the system selects key precision per block based on each block's importance to the current query.  Blocks holding the majority of the attention mass use FP16 keys (paged in from system RAM); the remaining blocks use INT8 keys with INT4 values stored in VRAM.  The selection adapts to the actual attention pattern at each decode step, providing FP16-quality scoring on the blocks that matter while retaining quantized efficiency on the long tail.

\subsection{Contributions}

\begin{enumerate}[leftmargin=2em,itemsep=4pt]
  \item A \textbf{tiered KV cache architecture} that stores per-channel INT8 keys and per-group INT4 values in VRAM (approximately 56\% of FP16 including metadata) while retaining FP16 originals in system RAM for runtime fallback.

  \item A \textbf{two-term error decomposition} (key compression error + value compression error) with independent, per-head, per-step bounds computable at runtime.

  \item An \textbf{adaptive precision selection mechanism} that uses the attention mass distribution (computed from quantized scores) to determine which blocks require FP16 keys, providing a guaranteed bound on key compression error.

  \item A \textbf{ranking-consistency check} that detects when INT8 scoring noise distorts the attention distribution, triggering per-head fallback to dense attention.  This mechanism is critical: it converts a previously unaddressed silent failure mode into a detectable and recoverable event when non-uniform precision would otherwise dilute attention on retrieval-critical tokens.

  \item A \textbf{four-rung fallback ladder} that escalates through increasingly conservative precision modes, with full FP16 recomputation as the unconditional terminal state.
\end{enumerate}

\textbf{What this paper does \emph{not} claim.}  We do not claim end-to-end model-quality guarantees.  The error bounds are \emph{per-head, per-step}: they bound $\norm{O_{\mathrm{quant}} - O_{\mathrm{ref}}}$ (quantized output vs.\ same kernel with FP16 KV; Section~\ref{sec:decomposition}) for a single head at a single decode step.  On fallback, the system returns $O_{\mathrm{dense}}$ directly.  The aggregate effect on model quality is assessed empirically.

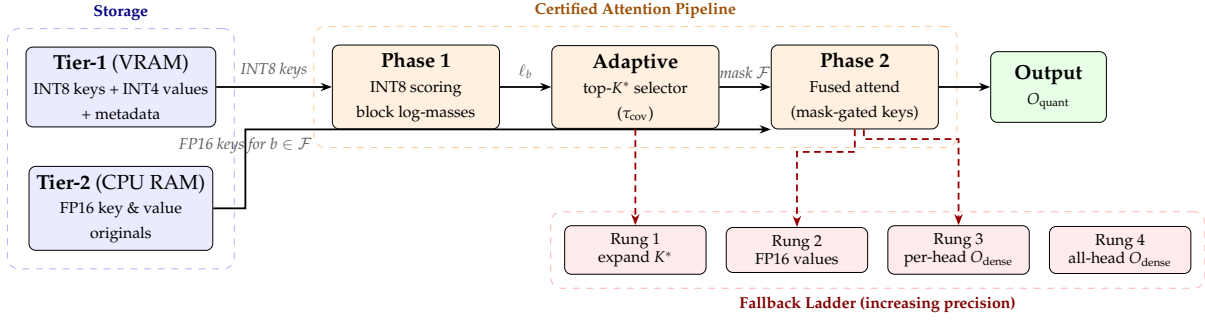
\begin{figure*}[!t]
\centering
\resizebox{\textwidth}{!}{%
\begin{tikzpicture}[
  box/.style={draw, rounded corners=3pt, minimum height=1.1cm,
              align=center, font=\small},
  tier/.style={box, fill=blue!8, minimum width=2.8cm},
  phase/.style={box, fill=orange!12, minimum width=2.6cm},
  rung/.style={box, fill=red!8, minimum width=2.2cm, minimum height=0.7cm,
               font=\scriptsize, align=center},
  arrow/.style={-{Stealth[length=5pt]}, thick},
  lbl/.style={font=\scriptsize\itshape, text=gray!70!black},
]

\node[tier] (t1) {\textbf{Tier-1} (VRAM)\\\scriptsize INT8 keys + INT4 values\\\scriptsize + metadata};
\node[tier, below=0.6cm of t1] (t2) {\textbf{Tier-2} (CPU RAM)\\\scriptsize FP16 key \& value\\\scriptsize originals};

\node[phase, right=1.8cm of t1] (p1) {\textbf{Phase 1}\\\scriptsize INT8 scoring\\\scriptsize block log-masses};
\node[phase, right=0.8cm of p1] (sel) {\textbf{Adaptive}\\\scriptsize top-$K^*$ selector\\\scriptsize ($\tau_{\mathrm{cov}}$)};
\node[phase, right=0.8cm of sel] (p2) {\textbf{Phase 2}\\\scriptsize Fused attend\\\scriptsize (mask-gated keys)};

\node[box, fill=green!10, right=0.8cm of p2, minimum width=1.8cm] (out) {\textbf{Output}\\\scriptsize $O_{\mathrm{quant}}$};

\node[rung, below=1.5cm of sel] (r1) {Rung 1\\expand $K^*$};
\node[rung, right=0.3cm of r1] (r2) {Rung 2\\FP16 values};
\node[rung, right=0.3cm of r2] (r3) {Rung 3\\per-head $O_{\mathrm{dense}}$};
\node[rung, right=0.3cm of r3] (r4) {Rung 4\\all-head $O_{\mathrm{dense}}$};

\draw[arrow] (t1) -- (p1) node[midway,above,lbl] {INT8 keys};
\draw[arrow] (p1) -- (sel) node[midway,above,lbl] {$\ell_b$};
\draw[arrow] (sel) -- (p2) node[midway,above,lbl] {mask $\mathcal{F}$};
\draw[arrow] (p2) -- (out);
\draw[arrow] (t2.east) -- ++(0.5,0) |- (p2.south west)
  node[pos=0.5, below, lbl] {FP16 keys for $b \in \mathcal{F}$};

\draw[arrow, densely dashed, red!60!black]
  (sel.south) -- ++(0,-0.55) -| (r1.north);
\draw[arrow, densely dashed, red!60!black]
  (p2.south) -- ++(0,-0.35) -| (r2.north);
\draw[arrow, densely dashed, red!60!black]
  ([xshift=4pt]p2.south) -- ++(0,-0.15) -| (r3.north);

\begin{scope}[on background layer]
  \node[draw=blue!30, dashed, rounded corners=5pt,
        fit=(t1)(t2), inner sep=8pt,
        label={[font=\scriptsize\bfseries,blue!50!black]above:Storage}] {};
  \node[draw=orange!40, dashed, rounded corners=5pt,
        fit=(p1)(sel)(p2), inner sep=8pt,
        label={[font=\scriptsize\bfseries,orange!60!black]above:Certified Attention Pipeline}] {};
  \node[draw=red!30, dashed, rounded corners=5pt,
        fit=(r1)(r2)(r3)(r4), inner sep=6pt,
        label={[font=\scriptsize\bfseries,red!50!black]below:Fallback Ladder (increasing precision)}] {};
\end{scope}
\end{tikzpicture}
}
\caption{System overview.  Tier-1 stores compressed KV in VRAM; Tier-2 retains FP16 originals in CPU RAM\@.  Phase~1 computes INT8 block log-masses; the adaptive selector chooses top-$K^*$ blocks for FP16 key promotion; Phase~2 fuses the mask-gated attend with INT4 value dequantization.  The fallback ladder (dashed red) escalates precision when monitors detect bound violations, with Rung~4 returning $O_{\mathrm{dense}}$ unconditionally.}
\label{fig:overview}
\end{figure*}

\section{Background and Notation}
\label{sec:background}

\subsection{Attention in Autoregressive Decoding}

At decode step $t$, the model computes single-query attention over the cached key-value pairs:
\begin{equation}
\label{eq:attention}
  O = \mathrm{softmax}\!\left(\frac{q \cdot K^\top}{\sqrt{d}}\right) V
\end{equation}
where $q \in \R^d$ is the current query, $K \in \R^{N \times d}$ the cached keys, $V \in \R^{N \times d}$ the cached values, and $d$ the head dimension.  This requires reading $2Nd$ elements from DRAM per head per step---the dominant cost at long context.

\subsection{Block Organisation}

The KV cache is partitioned into contiguous blocks of $B$ tokens (typically $B = 16$).  Block $b$ contains keys $\{k_t\}_{t \in b}$ and values $\{v_t\}_{t \in b}$.  All precision decisions and error monitoring operate at block granularity.  The block is the atomic unit of page-in from system RAM.

\subsection{Quantization Schemes}
\label{sec:quant_schemes}

\textbf{Per-channel INT8 keys.}  Each of the $d$ key channels is quantized independently with its own scale $\sigma_c$ and real-valued offset $z_c$:
\begin{equation}
  k_{t,c}^{\mathrm{int8}} = \mathrm{clamp}\!\left(\mathrm{round}\!\left(\frac{k_{t,c} - z_c}{\sigma_c}\right),\, -128,\, 127\right),
\end{equation}
\begin{equation}
  \hat{k}_{t,c} = k_{t,c}^{\mathrm{int8}} \cdot \sigma_c + z_c.
\end{equation}
For a completed block with channel minimum $\ell_c$ and maximum $u_c$, the affine parameters are chosen as $\sigma_c=(u_c-\ell_c)/255$ and $z_c=\ell_c+128\sigma_c=u_c-127\sigma_c$ (with the degenerate constant-channel case handled by an implementation guard).  Thus every value used to fit the block is representable without clipping, and round-to-nearest gives the deterministic reconstruction bound $|k_{t,c}-\hat{k}_{t,c}|\leq \sigma_c/2$.

Quantization is applied to keys \emph{after} RoPE (rotary position embeddings) has been applied, since LLaMA-family models store post-RoPE keys in the KV cache.  The quantization error bound $\Delta$ therefore accounts for any position-dependent magnitude variation introduced by RoPE\@.

Per-channel quantization is critical for keys.  Key channels have widely varying magnitudes (dynamic ranges spanning two orders of magnitude within a single head are common).  Per-block quantization (one scale for all channels) forces low-magnitude channels to share a scale set by the high-magnitude channels, crushing them to 1--2 INT8 levels.  This produces measurable quality degradation on language modelling benchmarks.  Per-channel quantization, established by KVQuant~\cite{kvquant} and adopted as standard practice, eliminates this by allowing each channel its own scale.

The per-channel scales and offsets are \emph{per-block}: each block of $B$ tokens has its own set of $2d$ FP32 values (scale + offset), computed once when the block is filled and never updated thereafter.  Appended tokens accumulate in a trailing partial block that remains in FP16 until $B$ tokens have arrived; at that point the block is quantized atomically with scales derived from those $B$ tokens alone.  This design eliminates the stale-scale problem: old blocks' INT8 codes always match their stored metadata, and the quantization error $\Delta_b$ is block-local.  The per-block metadata costs $2d \times 4 / B = 64$ bytes per token at $d{=}128$, $B{=}16$, adding 50\% to the raw INT8 key storage.  Total key-side Tier-1 storage is $128 + 64 = 192$ bytes per token, or 75\% of FP16 keys (256 bytes per token).  The INT8 codes alone are half the size of FP16, but the per-channel metadata narrows the effective compression.

\textbf{Per-group INT4 values.}  Each value vector of dimension $d$ is divided into groups of $g$ elements (default $g = 16$).  Each group is independently quantized to INT4 with its own FP16 scale and real-valued offset.  Two INT4 values are packed into a single byte.

Per-group granularity is necessary for values because value vectors exhibit higher intra-vector dynamic range variation than keys.  A single scale for the full $d{=}128$ vector would crush groups with small magnitudes.  The group size $g$ controls a quality--storage tradeoff: smaller groups provide finer-grained quantization at the cost of more metadata.  Empirically, $g{=}16$ achieves near-lossless \emph{perplexity} ($\Delta\text{ppl} = +0.03$) with zero fallback escalation at the default tolerance; however, value-sensitive tasks such as variable tracking and word extraction show measurable degradation at short contexts (Section~\ref{sec:ruler_8k_ablation}).  Smaller groups ($g{=}4, 8$) give the same quality but worse storage; larger groups ($g{=}32$) save marginally more memory but degrade sharply (Section~\ref{sec:results}).

The storage cost is $d/2$ bytes per token for the quantized values (INT4 packed), plus $2 \times (d/g) \times 2$ bytes per token for FP16 scales and offsets.  For $d{=}128$ and $g{=}16$: 64 bytes (values) + 32 bytes (metadata) = 96 bytes per token, compared to 256 bytes for FP16.  This is 37.5\% of FP16.

\textbf{Value error annotation.}  At quantization time (cache write), the system computes the per-token $\ell_2$ reconstruction error $\norm{\hat{V}_t - V_t}_2$ and reduces it to a per-block maximum $\eta_b$.  It also stores the per-block maximum value norm $\nu_b = \max_{t \in b}\norm{V_t}_2$; the key error bound uses $V_{\max} = \max_b \nu_b$.  Both annotations are one float per block, stored alongside the block in VRAM\@.  Computing $\eta_b$ and $\nu_b$ at write time avoids the circular dependency of needing FP16 originals to check quality at attend time.

\textbf{FP16 originals.}  The unquantized FP16 keys and values are retained in pinned system RAM after quantization.  They are the ground truth for the attention computation and serve as the unconditional fallback when quantized precision is insufficient.

\section{System Architecture}
\label{sec:arch}

\subsection{Tiered KV Cache}
\label{sec:tiered}

The KV cache is stored across two tiers, each serving a distinct role:

\textbf{Tier 1 --- VRAM (hot).}  Per-channel INT8 keys and per-group INT4 values, co-located with their quantization metadata (per-channel scales/offsets for keys, per-group FP16 scales/offsets for values) and per-block value error annotations $\eta_b$.  This is the data consumed during normal attention execution.  Total storage per token is detailed in Section~\ref{sec:storage}.

\textbf{Tier 2 --- CPU pinned RAM (cold).}  FP16 keys and values.  The ground truth.  Accessed when the adaptive precision selector promotes blocks to FP16 key precision, or when the value error annotation $\eta_b$ indicates that a block's INT4 values exceed the configured tolerance.  Page-in latency varies by interconnect: ${\sim}3$\,$\mu$s per block over PCIe~5.0, lower over NVLink or CXL, effectively zero on unified-memory architectures.  Optionally, a small VRAM scratch cache of recently paged-in FP16 blocks can amortise repeated page-in across consecutive decode steps; this does not affect correctness but reduces host-to-device transfer traffic in practice.  Storage per token: 100\% of FP16.

The key architectural decision is retaining Tier~2.  Conventional systems discard FP16 originals after quantization; by retaining them in system RAM, the system gains an unconditional fallback mechanism at the cost of $\approx 16$\,GB at 128K context (Section~\ref{sec:limitations}).

\subsection{Fused Quantized Attention Kernel}
\label{sec:kernel}

The attention kernel executes directly on the quantized data in Tier 1, performing in-register dequantization rather than staging through an FP16 buffer.  This is the ``compressed-domain execution'' approach introduced by HACK~\cite{hack} and VecInfer~\cite{vecinfer}: the INT8 keys are dequantized to FP16 in registers during the dot product, and the INT4 values are dequantized in registers during the weighted sum.  No intermediate FP16 buffer is allocated.

For a single decode step:
\begin{enumerate}[leftmargin=2em,itemsep=2pt]
  \item For each block $b$, read INT8 keys from VRAM ($B \times d$ bytes).
  \item Dequantize each key channel in registers: $\hat{k}_{t,c} = k_{t,c}^{\mathrm{int8}} \cdot \sigma_c + z_c$.
  \item Compute attention logits: $s_{b,t} = q \cdot \hat{k}_{b,t} / \sqrt{d}$.
  \item Compute block statistics: $m_b = \max_t s_{b,t}$ and per-block unnormalised mass.
  \item Read INT4 values from VRAM, dequantize per-group in registers.
  \item Accumulate into online softmax state: weights $\times$ values.
\end{enumerate}

The scoring pass over all blocks produces per-block mass estimates used by the adaptive precision selector (Section~\ref{sec:adaptive}).  The selector determines which blocks require FP16 key precision; FP16 keys for those blocks are paged in from Tier~2 before the attend pass begins.

\textbf{Fused scoring and attend.}  The system executes in two phases: a lightweight scoring pass to compute block masses and drive the adaptive selector, followed by a single fused attend pass that processes all blocks with mask-gated key precision.  Algorithm~\ref{alg:fused} gives the exact logic.

\begin{figure*}[!t]
\small
\begin{center}
\fbox{\parbox{0.92\textwidth}{
\textbf{Algorithm 1: Fused Certified Attention with Mask-Gated Key Precision} \\[4pt]
\textbf{Input:} Query $q$; blocks $\{b_1, \ldots, b_N\}$ with INT8 keys + INT4 values in VRAM; top-$K^*$ mask $\mathcal{F}$ from adaptive selector; FP16 keys for $\mathcal{F}$ in paged scratch buffer \\
\textbf{Output:} Attention output $O$ \\[4pt]
\textit{// --- Phase 1: INT8 scoring (block log-masses for adaptive selector) ---} \\
\textbf{for} each block $b$ \textbf{do} \\
\quad Dequantize INT8 keys in-register: $\hat{k}_{t,c} = k_{t,c}^{\mathrm{int8}} \cdot \sigma_c + z_c$ \\
\quad Compute logits: $s_{b,t} = q \cdot \hat{k}_{b,t} / \sqrt{d}$ \\
\quad $m_b \leftarrow \max_t s_{b,t}$; \quad $S_b \leftarrow \sum_{t \in b} \exp(s_{b,t} - m_b)$ \\
\quad $\ell_b \leftarrow m_b + \log S_b$ \quad \textit{// block log-mass (globally comparable)} \\
\textbf{end for} \\[3pt]
\textit{// --- Adaptive selection + async page-in ---} \\
$p_b \leftarrow \exp(\ell_b - \mathrm{logsumexp}_j\, \ell_j)$ for all $b$ \quad \textit{// normalised block mass} \\
$\mathcal{F} \leftarrow \mathrm{AdaptiveTopK}(\{p_b\}, \tau_{\mathrm{cov}}, K_{\min}, K_{\max})$ \\
Page in FP16 keys for blocks $b \in \mathcal{F}$ from Tier~2 into scratch buffer \\[3pt]
\textit{// --- Phase 2: fused attend (single pass, mask-gated key precision) ---} \\
Initialise online softmax state: $m \leftarrow -\infty$ (FP32), $\ell \leftarrow 0$ (FP32), $o \leftarrow \mathbf{0}$ (FP32, $d$-dim) \\[2pt]
\textbf{for} each block $b$ \textbf{do} \\
\quad \textit{// Load both key representations} \\
\quad $\hat{k}^{\mathrm{int8}}_{b} \leftarrow$ dequantize INT8 keys for block $b$ \quad (always loaded) \\
\quad \textbf{if} $b \in \mathcal{F}$: $k^{\mathrm{fp16}}_{b} \leftarrow$ FP16 keys from scratch buffer \\[2pt]
\quad \textit{// Branchless precision selection} \\
\quad $k_b \leftarrow \begin{cases} k^{\mathrm{fp16}}_b & \text{if } b \in \mathcal{F} \\ \hat{k}^{\mathrm{int8}}_b & \text{otherwise} \end{cases}$ \quad (mask-gated \texttt{where}) \\[2pt]
\quad \textit{// Score and accumulate (standard online softmax)} \\
\quad $s_{b,t} \leftarrow q \cdot k_{b,t} / \sqrt{d}$ \\
\quad $m_b^{\mathrm{new}} \leftarrow \max_t s_{b,t}$ \\
\quad Update online softmax: $m, \ell, o$ with block scores and INT4 values $\hat{V}_b$ \\
\textbf{end for} \\[3pt]
$O \leftarrow o \,/\, \ell$
}}
\end{center}
\caption{Fused certified attention with mask-gated key precision.  Phase~1 is a lightweight scoring pass (INT8 keys only, no value reads) that computes per-block log-masses $\ell_b$ for the adaptive selector.  Phase~2 is a single fused attend pass: each block's keys are selected from FP16 (promoted) or INT8 (rest) via mask-gated \texttt{where}.  Values are INT4 for non-promoted blocks (Rung~2 promoted blocks use paged-in FP16 values).  The online softmax state ($m$, $\ell$, $o$) uses FP32 precision throughout (Section~\ref{sec:numerical}).  Ranking-consistency checks occur after this kernel; value error $E_{\mathrm{val}}$ is computed post-attend from cached attention masses and pre-stored $\eta_b$.}
\label{alg:fused}
\end{figure*}

\subsection{Adaptive Precision Selection}
\label{sec:adaptive}

After the scoring pass (Phase~1 in Algorithm~\ref{alg:fused}), the system has per-block mass estimates from the INT8 scores.  These drive the precision selection for the attend pass:

\begin{enumerate}[leftmargin=2em,itemsep=2pt]
  \item Sort all blocks by estimated attention mass (descending).
  \item Compute cumulative mass fraction.
  \item Find $K^* = \min\{k : C_k \geq \tau_{\mathrm{cov}}\}$, where $\tau_{\mathrm{cov}}$ is the coverage threshold (default 0.995).
  \item Clamp: $K^* = \mathrm{clamp}(K^*, K_{\min}, K_{\max})$.
  \item Page in FP16 keys from Tier 2 for the top $K^*$ blocks into a compact scratch buffer.
  \item Execute the fused attend pass (Phase~2): each block uses either its FP16 keys (if promoted) or INT8 keys (otherwise), selected via a branchless mask.
\end{enumerate}

The coverage threshold $\tau_{\mathrm{cov}}$ controls the selector: blocks whose \emph{INT8-estimated} masses sum to at least $\tau_{\mathrm{cov}}$ of the total estimated mass are promoted to FP16 keys.  The certified guarantee (Theorem~\ref{thm:mass}) is on the \emph{tail}: the true mass left to INT8 execution is at most $e^{2\Delta}(1 - \tau_{\mathrm{cov}})$, regardless of whether the ranking is exactly correct.  The implementation conservatively substitutes $e^{3\Delta}$ (Section~\ref{sec:scoring}).

The coverage threshold $\tau_{\mathrm{cov}} = 0.995$ means $K^*$ adapts to the attention pattern: on concentrated heads (e.g., a single sink token holding most of the mass), $K^*$ can be as low as $K_{\min}$; on diffuse heads (e.g., broad context gathering in early layers), $K^*$ expands toward the cap.  With $K_{\max}{=}128$ blocks, the maximum promoted fraction is 50\% at 4K (256 blocks), 25\% at 8K (512 blocks), and 12.5\% at 16K (1024 blocks).  Empirically, $K^*$ frequently saturates at the cap for short contexts, where $\tau_{\mathrm{cov}}{=}0.995$ demands most blocks for coverage, and falls well below the cap at longer contexts where attention mass concentrates on a smaller fraction of blocks (Section~\ref{sec:results}).  The cost is self-regulating: the system spends precision where the attention pattern demands it.

\subsection{Fallback Ladder}
\label{sec:fallback}

When runtime monitoring detects that an error bound may be exceeded, the system escalates through increasingly conservative precision modes:

\begin{enumerate}[leftmargin=2em,itemsep=2pt]
  \item \textbf{Expand FP16 key coverage}: Increase $K^*$ beyond the adaptive selection (e.g., double it), reducing the tail mass on INT8 keys.
  \item \textbf{Promote values to FP16}: Before Phase~2, the default policy pages in FP16 values for blocks whose estimated contribution $\hat{\rho}_b \cdot \eta_b$ (Phase-1 estimated mass $\times$ pre-stored reconstruction error) exceeds $v_{\mathrm{tol}}$ (Corollary~\ref{cor:val_block}).  This is a local promotion heuristic, not a claim that the global value term is below $v_{\mathrm{tol}}$.  After Phase~2, the exact $E_{\mathrm{val}} = \sum_b \rho_b \eta_b$ is computed from actual attention masses and reported as telemetry; promoted blocks contribute $\eta_b{=}0$, so no recomputation is required.  Deployments that require a hard value-error budget can instead promote blocks greedily, in descending $\hat{\rho}_b\eta_b$, until the achieved runtime $E_{\mathrm{val}}$ is below the configured budget.
  \item \textbf{Ranking-consistency check}: After Phase~2, compare the top-$r$ block ranking under FP16 block log-masses against the INT8 ranking from Phase~1, and verify that no tail block's upper-bounded FP16 log-mass could enter the top-$r$ (Eq.~\ref{eq:boundary_check}).  If either check fails (Section~\ref{sec:ranking_consistency}), recompute attention for the affected head using all-FP16 keys and values from Tier~2.  This per-head fallback guarantees dense-equivalent output on the steps where INT8 scoring noise would distort the attention distribution.
  \item \textbf{Full FP16 recomputation}: Page in all FP16 keys and values from Tier~2; execute standard dense FP16 attention for all heads via \torchsdpa{} (the same code path used by the Dense baseline).
\end{enumerate}

The bounds certify error magnitude; the escalation policy uses heuristic thresholds for efficiency rather than enforcing a globally optimal promotion policy.

Rung~3 is per-head: the ranking-consistency check fires for a single head, and only that head's output is recomputed via \torchsdpa{} on all-FP16 KV\@.  Rung~4 is per-layer, all-heads: the full FP16 K/V for the current layer's active context is paged from pinned CPU memory to GPU, \torchsdpa{} is called for all heads in that layer, and the layer output is returned early.  Both rungs bypass the certified attention kernel entirely and return $O_{\mathrm{dense}}$ (Section~\ref{sec:decomposition}).

\textbf{Fallback memory contract.}  Rung~4 is layerwise, not a persistent full dense cache: only one layer's FP16 K/V tensors are staged into GPU memory at a time, and the temporary buffers are released after the SDPA call.  The peak temporary VRAM for a Rung~4 event is therefore one layer's full-context FP16 K/V plus SDPA workspace---e.g., ${\sim}0.5$\,GB at 128K context on LLaMA~3.1-8B (all 8 KV heads for one layer: $2 \times 128\text{K} \times 128 \times 8 \times 2$\,bytes)---not $\text{num\_layers} \times$ full FP16 KV\@.  The ``unconditional'' qualifier refers to correctness, not resource-free: Rung~4 requires sufficient transient GPU memory for one layer's staging plus PCIe bandwidth for the page-in, but does not require the dense KV cache to be persistently resident in VRAM\@.

The practical value of the system depends on Rungs~3 and~4 being rarely triggered, which is the subject of the empirical evaluation.  The runtime cost of escalation is host-to-device page-in latency, overlappable with GPU compute on other blocks.

\section{Formal Error Bounds}
\label{sec:bounds}

\subsection{Two-Term Error Decomposition}
\label{sec:decomposition}

\paragraph{Intuition.}
We bound two independent failure modes: (1) distortion of the attention distribution from key quantization, and (2) reconstruction error of values.  Each is monitored separately at runtime and each has a distinct fallback path.

\paragraph{Reference outputs (summary).}
The paper distinguishes three attention outputs, each computed for a single head at a single decode step:

\begin{itemize}[leftmargin=2em,itemsep=2pt]
\item $O_{\mathrm{dense}}$: the output of \torchsdpa{} (Flash Attention) with FP16 keys and values.  This is the production Dense baseline.
\item $O_{\mathrm{ref}}$: the output of the \emph{same certified attention kernel} used by the certified system, but operating on unquantized FP16 keys and values.  This isolates quantization error from the arithmetic-path difference between the certified kernel and Flash Attention (tensor-core matmul).
\item $O_{\mathrm{quant}}$: the output of the certified kernel operating on quantized (INT8 key / INT4 value) data.  This is the system's fast-path output.
\end{itemize}

All formal bounds in this section certify $\norm{O_{\mathrm{quant}} - O_{\mathrm{ref}}}_2$, i.e.\ quantization error with the arithmetic path held fixed.  The arithmetic-path gap $\norm{O_{\mathrm{ref}} - O_{\mathrm{dense}}}_2$ is characterised separately in Section~\ref{sec:numerical}.  On fallback (Rungs~3 and~4), the system bypasses the certified kernel entirely and returns $O_{\mathrm{dense}}$, so the worst case is the Dense baseline---not merely $O_{\mathrm{ref}}$.

The total quantization error decomposes into two independent terms via the triangle inequality through the bridge term $\sum_t a_t' V_t$ (quantized-key weights, exact values):
\begin{align}
\label{eq:decomp}
  &\norm{O_{\mathrm{quant}} - O_{\mathrm{ref}}}_2
  = \norm{\textstyle\sum_t a_t' \hat{V}_t - \sum_t a_t V_t}_2 \notag \\
  &\quad\leq \underbrace{\norm{\textstyle\sum_t a_t' V_t - \sum_t a_t V_t}_2}_{E_{\mathrm{key}}} \notag \\
  &\quad\;+ \underbrace{\norm{\textstyle\sum_t a_t' \hat{V}_t - \sum_t a_t' V_t}_2}_{E_{\mathrm{val}}}
\end{align}
where $a_t$ are the FP16-key softmax weights, $a_t'$ the quantized-key weights, $V_t$ the FP16 values, and $\hat{V}_t$ the quantized values.  $E_{\mathrm{key}}$ measures the effect of key quantization on the attention distribution with exact values; $E_{\mathrm{val}}$ measures value reconstruction error under the quantized-key distribution.

Each term has an independent bound, an independent runtime monitor, and an independent escalation path through the fallback ladder.  This independence is practically useful: key compression affects the \emph{attention distribution} (which tokens are attended to), while value compression affects the \emph{output quality} given a fixed distribution.  The two error sources have different characteristics and different remedies.

\subsection{Theorem 1: Value Compression Error}
\label{sec:val_error}

\begin{theorem}[Value Compression Error]
\label{thm:val}
Let $a_t$ be the softmax attention weights (computed with whatever key precision is in use).  Let $V_t$ be the original FP16 value and $\hat{V}_t$ the quantized value, with $\norm{V_t - \hat{V}_t}_2 \leq \eta$ for all $t$ in the attended set.  Then:
\begin{equation}
  E_{\mathrm{val}} = \norm{\sum_t a_t \hat{V}_t - \sum_t a_t V_t}_2 \leq \eta
\end{equation}
\end{theorem}

\begin{proof}
By the triangle inequality and the convexity of norms:
\[
  \norm{\sum_t a_t (\hat{V}_t - V_t)}_2 \leq \sum_t a_t \norm{\hat{V}_t - V_t}_2 \leq \eta \sum_t a_t = \eta
\]
where the last equality uses $\sum_t a_t = 1$.
\end{proof}

\begin{corollary}[Blockwise value compression error]
\label{cor:val_block}
Let $\rho_b = \sum_{t \in b} a_t$ be the attention mass on block $b$, and $\eta_b = \max_{t \in b} \norm{V_t - \hat{V}_t}_2$ the per-block value error annotation.  Then:
\begin{equation}
  E_{\mathrm{val}} \leq \sum_b \rho_b \,\eta_b
  \label{eq:val_block}
\end{equation}
\end{corollary}

This blockwise form is computed exactly at runtime: the system evaluates
$E_{\mathrm{val}}^{(h)} = \sum_b \rho_b^{(h)} \eta_b$ per head $h$ at each decode step, using block attention masses $\rho_b^{(h)}$ accumulated during the Phase-2 online softmax pass and the stored per-block error annotations $\eta_b$.  Blocks where $\rho_b \,\eta_b > v_{\mathrm{tol}}$ are promoted to FP16 values from Tier-2, and the achieved $E_{\mathrm{val}}^{(h)}$ after promotion is reported as part of the per-step certificate.  A stronger rule would promote blocks greedily by descending $\rho_b \,\eta_b$ until $E_{\mathrm{val}}^{(h)} \leq \varepsilon_{\mathrm{val}}$ (total value-error budget enforcement); we note this as a straightforward extension but retain the per-block threshold as the current escalation policy.

\textbf{Properties of this bound.}

The bound depends only on the worst-case per-token reconstruction error $\eta$, not on the number of tokens $N$.  This is because the softmax weights form a convex combination: the weighted average of bounded errors is itself bounded, regardless of how many terms participate.

For INT4 per-group values with group size $g{=}16$, $\eta$ is determined by the maximum within-group dynamic range across all value tokens.  Empirically, $\eta$ is small and stable across prompts (mean $\sim$0.05, max $\sim$0.18 for random Gaussian-distributed values).

The bound is tight: it is achieved when all per-token compression errors are perfectly aligned (pointing in the same direction).  In practice, compression errors across tokens are approximately random in direction, so the actual error is much smaller than $\eta$.

\textbf{Escalation path.}  When the quant-time error annotation $\eta_b$ exceeds the configured threshold for a block (e.g., due to an outlier value token), the system pages in FP16 values from Tier~2 for the affected block.

\subsection{Theorem 2: Key Compression Error}
\label{sec:key_error}

Key compression error is more subtle than value compression error because keys affect the attention distribution nonlinearly through the softmax.

\begin{theorem}[Key Compression Error]
\label{thm:key}
Let $a_t$ be the exact softmax weights (from FP16 keys) and $a_t'$ the approximate softmax weights (from quantized keys).  Let $V_{\max} = \max_t \norm{V_t}_2$.  Then:
\begin{equation}
  E_{\mathrm{key}} = \norm{\sum_t a_t' V_t - \sum_t a_t V_t}_2 \leq 2 V_{\max} \cdot \mathrm{TV}(a, a')
\end{equation}
where $\mathrm{TV}(a, a') = \frac{1}{2}\sum_t |a_t - a_t'|$ is the total variation distance between the true and approximate attention distributions.
\end{theorem}

\begin{proof}
\begin{align*}
  \norm{\sum_t (a_t' - a_t) V_t}_2
    &\leq \sum_t |a_t' - a_t| \cdot \norm{V_t}_2 \\
    &\leq V_{\max} \sum_t |a_t' - a_t| \\
    &= 2 V_{\max} \cdot \mathrm{TV}(a, a')
\end{align*}
\end{proof}

\textbf{Bounding the total variation.}  The total variation between the exact and approximate softmax distributions depends on the per-token score errors $|\Delta_t| = |s_t - s_t'|$.  We derive the bound for per-channel INT8 quantization with channel scales $\{\sigma_c\}$.

The per-channel quantization error satisfies $|k_{t,c} - \hat{k}_{t,c}| \leq \sigma_c / 2$ for each channel $c$.  The score error is:
\begin{align*}
  |\Delta_t| &= \frac{|q \cdot (k_t - \hat{k}_t)|}{\sqrt{d}} \leq \frac{1}{\sqrt{d}} \sum_{c=1}^d |q_c| \cdot \frac{\sigma_c}{2}
\end{align*}
Applying the Cauchy--Schwarz inequality to $\sum_c |q_c| \sigma_c \leq \norm{q}_2 \cdot \norm{\boldsymbol{\sigma}}_2$ where $\boldsymbol{\sigma} = (\sigma_1, \ldots, \sigma_d)$:
\begin{equation}
\label{eq:delta}
  |\Delta_t| \leq \frac{\norm{q}_2 \cdot \norm{\boldsymbol{\sigma}}_2}{2\sqrt{d}} \triangleq \Delta
\end{equation}

Since scales are per-block (Section~\ref{sec:quant_schemes}), $\Delta$ is strictly per-block: each block $b$ has its own $\Delta_b = \norm{q}_2 \cdot \norm{\boldsymbol{\sigma}^{(b)}}_2 / (2\sqrt{d})$ where $\boldsymbol{\sigma}^{(b)}$ is that block's scale vector.  At runtime, a tighter per-block bound can be computed directly as $\Delta_b = \frac{1}{2\sqrt{d}} \sum_c |q_c| \sigma_c^{(b)}$.  For all downstream bounds (Theorems~\ref{thm:mass} and~\ref{thm:key}, Eq.~\ref{eq:total}), we use the conservative $\Delta_h = \max_{b} \Delta_b$ over all fully-quantized Phase-1--scored blocks for head $h$ (excluding any trailing partial block that remains in FP16).  This is a single scalar per head per step and upper-bounds every block's individual $\Delta_b$.  Crucially, this maximum is taken over \emph{all} scored blocks---not only the tail set $\mathcal{T}$---because the normalised mass bound (Theorem~\ref{thm:mass}) involves a denominator over all blocks.

For per-channel quantization, $\norm{\boldsymbol{\sigma}}_2 = \sqrt{\sum_c \sigma_c^2}$.  When channel scales vary (as they do in practice), this is tighter than the per-block bound $\sqrt{d} \cdot \sigma_{\max}$ because channels with small scales contribute less.  For uniform scales ($\sigma_c = \delta$ for all $c$), the bound reduces to $\Delta = \norm{q}_2 \cdot \delta / 2$, recovering the standard per-block result.

\textbf{Bounding TV from score perturbation.}  The total variation between two softmax distributions whose logits differ by at most $\Delta$ can be bounded directly (Appendix~\ref{app:tv_proof}, Lemma~\ref{lem:tv}).  For softmax distributions $a_t = \exp(s_t) / Z$ and $a_t' = \exp(s_t') / Z'$:
\begin{equation}
\label{eq:tv_bound}
  \mathrm{TV}(a, a') \leq \tanh(\Delta) = \frac{e^{2\Delta} - 1}{e^{2\Delta} + 1}
\end{equation}
This bound is tight, independent of sequence length $N$, and satisfies $\tanh(\Delta) \approx \Delta$ for small $\Delta$.

In the adaptive precision system, the TV bound applies only to the \emph{tail} blocks (those using INT8 keys).  Let $\mathcal{F}$ denote the FP16 set and $\mathcal{T}$ the INT8 tail set.  Since the FP16 blocks have zero key quantization error, the key compression error decomposes as:
\begin{equation}
  E_{\mathrm{key}} \leq 2 V_{\max} \cdot \mathrm{TV}_{\mathcal{T}}(a, a')
\end{equation}
where $\mathrm{TV}_{\mathcal{T}}$ is the total variation restricted to the tail blocks.  The tail blocks hold at most $(1 - \tau_{\mathrm{cov}})$ of the \emph{estimated} mass (by construction of the selector), and Theorem~\ref{thm:mass} ensures their true mass is at most $e^{2\Delta}(1 - \tau_{\mathrm{cov}})$.  The TV on the tail is therefore bounded by the product of the tail's mass fraction and the per-token score perturbation within the tail:
\begin{equation}
  E_{\mathrm{key}} \leq 2 V_{\max} \cdot e^{2\Delta} (1 - \tau_{\mathrm{cov}}) \cdot (e^{2\Delta} - 1)
\end{equation}
At $\tau_{\mathrm{cov}} = 0.995$ and $\Delta = 0.18$: $E_{\mathrm{key}} \leq 2 V_{\max} \cdot 1.43 \cdot 0.005 \cdot 0.43 \approx 0.006 V_{\max}$.  The implementation uses INT8$\times$INT8 tensor-core scoring (Section~\ref{sec:scoring}), which introduces a third $\Delta$ from scaled-query quantization.  The emitted implementation certificate substitutes $e^{3\Delta}$ for $e^{2\Delta}$:
\begin{equation}
\label{eq:impl_key}
  E_{\mathrm{key}}^{\mathrm{impl}} \leq 2 V_{\max} \cdot e^{3\Delta} (1 - \tau_{\mathrm{cov}}) \cdot (e^{2\Delta} - 1)
\end{equation}
At the same operating point this gives $\approx 0.007 V_{\max}$---a 17\% wider bound that accounts for the scoring-path quantization.

\textbf{Escalation path.}  When the computed $E_{\mathrm{key}}$ exceeds a configured threshold, the system expands the FP16 set (increasing $\tau_{\mathrm{cov}}$ and reducing the tail mass) or escalates through the fallback ladder.

\subsection{INT8 Score Error for Mass Estimation}
\label{sec:score_error}

The adaptive precision selector uses INT8 scores to estimate per-block attention mass.  We need to bound the error in these estimates to ensure the selector identifies the correct high-mass blocks.

\begin{theorem}[INT8 Mass Estimation]
\label{thm:mass}
Let $S_b = \sum_{t \in b} \exp(s_t' - m_b')$ and $m_b' = \max_{t \in b} s_t'$ be the block sum and block max computed from INT8 scores.  Let $m_{\mathrm{global}}' = \max_b m_b'$ be the INT8 global max.  Then the true (FP16) unnormalised mass of block $b$ satisfies:
\begin{equation}
\label{eq:mass_bound}
  M_b^{\mathrm{fp16}} \leq S_b \cdot \exp(m_b' - m_{\mathrm{global}}' + 2\Delta)
\end{equation}
where $\Delta$ is defined in~\eqref{eq:delta}.  The implementation uses $\exp(3\Delta)$ in place of the tight $\exp(2\Delta)$: the extra factor of $e^{\Delta}$ provides headroom for optional INT8$\times$INT8 tensor-core scoring of the scaled query (Section~\ref{sec:scoring}), which would introduce a third error source.  All theorems state the tight $\exp(2\Delta)$; the implementation's $\exp(3\Delta)$ substitution is noted where relevant.
\end{theorem}

\begin{proof}
We bound $M_b^{\mathrm{fp16}} = \sum_{t \in b} \exp(s_t - m_{\mathrm{global}})$ using INT8 quantities.  Two error sources contribute:

\emph{Factor 1 (score perturbation):}  $|s_t - s_t'| \leq \Delta$ for each token, so $\exp(s_t) \leq e^{\Delta} \exp(s_t')$.

\emph{Factor 2 (global-max shift):}  $m_{\mathrm{global}} \geq m_{\mathrm{global}}' - \Delta$, so $\exp(-m_{\mathrm{global}}) \leq e^{\Delta} \exp(-m_{\mathrm{global}}')$.

Combining:
\begin{align*}
  M_b^{\mathrm{fp16}} &= \sum_{t \in b}\exp(s_t - m_{\mathrm{global}}) \\
  &\leq e^{\Delta} \sum_{t \in b}\exp(s_t' - m_{\mathrm{global}}) \\
  &\leq e^{2\Delta} \sum_{t \in b}\exp(s_t' - m_{\mathrm{global}}')
\end{align*}
The identity $\sum_{t \in b}\exp(s_t' - m_{\mathrm{global}}') = S_b \cdot \exp(m_b' - m_{\mathrm{global}}')$ is exact (algebraic re-expression, not an approximation), yielding:
\[
  M_b^{\mathrm{fp16}} \leq e^{2\Delta} \cdot S_b \cdot \exp(m_b' - m_{\mathrm{global}}')
\]
\end{proof}

\textbf{Normalised tail implication.}  The selector uses normalised block masses, so the one-sided unnormalised statement above is used through the following ratio bound.  For any subset $T$ of tokens,
\begin{align}
  P(T) &= \frac{\sum_{t\in T}\exp(s_t)}{\sum_j \exp(s_j)}
  \leq e^{2\Delta}
  \frac{\sum_{t\in T}\exp(s'_t)}{\sum_j \exp(s'_j)} \notag \\
  &= e^{2\Delta}\hat{P}(T).
\end{align}
Thus, if the INT8-estimated tail mass after selection is $\hat{\alpha}_{\mathcal{T}}$, the corresponding true FP16 tail mass is bounded by $\alpha_{\mathcal{T}} \leq \min\{1,\; e^{2\Delta}\hat{\alpha}_{\mathcal{T}}\}$.  This is the normalised form used in Equation~\eqref{eq:total}.

\textbf{Numerical evaluation.}  For $d = 128$ (LLaMA~3.1-8B), typical per-channel quantization gives $\Delta \approx 0.18$.  The tight bound is $\exp(2\Delta) \approx 1.43$; the implementation's $\exp(3\Delta)$ substitution gives $\approx 1.72$.

\textbf{What the upper bound certifies and what it does not.}  Theorem~\ref{thm:mass} provides a one-sided \emph{upper} bound: the true mass of block $b$ is at most $R_b = e^{2\Delta} \cdot S_b \cdot \exp(m_b' - m_{\mathrm{global}}')$.  This certifies a useful aggregate property: if the selector chooses blocks whose INT8-estimated masses sum to at least $\tau_{\mathrm{cov}}$ of the total INT8-estimated mass, then the true mass of the \emph{unchosen} blocks is at most $e^{2\Delta} (1 - \tau_{\mathrm{cov}})$ of the total true mass.

However, the upper bound does \emph{not} certify that the INT8 ranking is correct---block $A$ with a higher INT8 estimate than block $B$ is not guaranteed to have higher true mass.  Both could have true masses anywhere within a multiplicative $e^{\pm 2\Delta}$ envelope around their INT8 estimates.  This means the selector may sometimes promote blocks that are not truly in the top-$K^*$.  The consequence is benign: FP16 keys are used for a block that did not strictly need them, wasting a page-in but not introducing error.  The guarantee that matters is on the \emph{tail}: the total true mass left to INT8 is bounded, regardless of which specific blocks are in the FP16 set.

The fallback ladder (Section~\ref{sec:fallback}) provides the unconditional backstop: if the error from INT8 execution on the tail exceeds any configured threshold, the system expands the FP16 set or reverts to full FP16.

\subsection{Total Error Bound}

Combining the two terms:
\begin{equation}
\label{eq:total}
  \norm{O_{\mathrm{quant}} \!-\! O_{\mathrm{ref}}}_2
  \leq \underbrace{2 V_{\max} e^{2\Delta}\hat{\alpha}_{\mathcal{T}}(e^{2\Delta} \!-\! 1)}_{E_{\mathrm{key}}}
  + \underbrace{\textstyle\sum_b \rho_b\eta_b}_{E_{\mathrm{val}}}
\end{equation}
where $\hat{\alpha}_{\mathcal{T}} = \sum_{b \notin \mathcal{F}} p_b^{\mathrm{int8}}$ is the INT8-estimated tail mass \emph{after} all clamping and Rung~1 expansion, and $E_{\mathrm{val}} = \sum_b \rho_b \eta_b$ is the achieved per-head value error (Corollary~\ref{cor:val_block}), computed exactly at runtime.  The adaptive selector targets $\hat{\alpha}_{\mathcal{T}} \leq 1 - \tau_{\mathrm{cov}}$, but $K_{\max}$ clamping may prevent this; the bound uses the achieved value, not the nominal threshold.  By Theorem~\ref{thm:mass}, the true tail mass satisfies $\alpha_{\mathcal{T}} \leq e^{2\Delta}\,\hat{\alpha}_{\mathcal{T}}$, so the bound is valid regardless of whether $\tau_{\mathrm{cov}}$ was fully achieved.  When the target is met ($\hat{\alpha}_{\mathcal{T}} = 1 - \tau_{\mathrm{cov}}$), the bound reduces to $E_{\mathrm{key}} \leq 2V_{\max} \cdot e^{2\Delta}(1 - \tau_{\mathrm{cov}})(e^{2\Delta} - 1)$.

At $\tau_{\mathrm{cov}} = 0.995$ and $\Delta = 0.18$, and assuming the target is met, the key term evaluates to $\approx 0.006\,V_{\max}$ ($\approx 0.007\,V_{\max}$ under the $e^{3\Delta}$ substitution).  Both $E_{\mathrm{key}}$ and $E_{\mathrm{val}}$ are computed at runtime, independently monitorable, and independently escalatable through the fallback ladder.

\subsection{Preconditions and Monitoring}
\label{sec:preconditions}

The error bound~\eqref{eq:total} is valid under four preconditions.  We distinguish between how each is \emph{guaranteed} and how each is \emph{monitored}:

\textbf{P1: Quantization metadata is current.}  Block scales and offsets must reflect the actual data in the block.  \emph{Guarantee:} deterministic by construction.  Per-block, per-channel scales are computed once at block-fill time (Section~\ref{sec:quant_schemes}) and are immutable thereafter; blocks are never partially overwritten.  Appended tokens remain in FP16 in a trailing partial block until $B$ tokens have arrived, at which point the block is quantized atomically.

\textbf{P2: Quantization error is bounded.}  The per-channel quantization error for each token must satisfy $|k_{t,c} - \hat{k}_{t,c}| \leq \sigma_c / 2$.  \emph{Guarantee:} deterministic for round-to-nearest quantization without saturation.  Per-block min/max scales and offsets are chosen at block-fill time as described in Section~\ref{sec:quant_schemes}, so all values used to fit the block are representable; saturation cannot occur unless the stored metadata is corrupted or stale.  \emph{Monitoring:} a small exploration budget (Section~\ref{sec:instability}) spot-checks dequantized values against FP16 originals as a defence-in-depth measure.

\textbf{P3: No arithmetic overflow.}  The scoring and softmax computations must not overflow the FP32 accumulator range.  \emph{Guarantee:} deterministic for FP32 accumulators at $d \leq 256$ and typical score magnitudes ($|s_t| \leq d \cdot \max|q| \cdot \max|k| / \sqrt{d} \ll 2^{127}$); verified by standard overflow checks.  All experiments in this paper use FP32 accumulators for the online-softmax scalars ($m$, $\ell$) and output accumulator $o$.

\textbf{P4: Tier-2 originals and staging buffers are available.}  The fallback guarantee requires the FP16 KV originals to be retained in pinned system RAM and sufficient transient VRAM staging capacity to run the dense fallback for the affected layer/head.  \emph{Guarantee:} by system design---Tier-2 is allocated at model load and never released during inference; staging capacity for one layer is ${\sim}0.5$\,GB at 128K context (Section~\ref{sec:fallback}).  If Tier-2 memory is exhausted or staging allocation fails, the system cannot fall back and must report an error.

When any precondition is violated, the system escalates through the fallback ladder.  Rung~4 (full FP16 recomputation via \torchsdpa{}) is always available, making the system correct by construction.  The overall guarantee is:

\begin{center}
\fbox{\parbox{0.88\linewidth}{%
\textbf{Contract.}\enspace Given preconditions P1--P4, for each head $h$ and decode step $t$, the system either:
\begin{enumerate}[leftmargin=1.5em,itemsep=1pt,topsep=2pt]
  \item returns $O_{\mathrm{quant}}$ satisfying $\norm{O_{\mathrm{quant}} - O_{\mathrm{ref}}}_2 \leq E_{\mathrm{key}} + E_{\mathrm{val}}$ (fast path, bounded relative to $O_{\mathrm{ref}}$); or
  \item returns $O_{\mathrm{dense}}$ via \torchsdpa{} (Rung~3 per-head, Rung~4 all-heads; fallback path, exact Dense baseline).
\end{enumerate}
The system never returns an output with unknown or unbounded error.  Both $E_{\mathrm{key}}$ and $E_{\mathrm{val}}$ are computed at runtime and available as per-step telemetry.  If P4 cannot be satisfied (Tier-2 memory exhausted), the system reports an error rather than returning an uncertified output.}}
\end{center}

\section{Compressed-Domain Execution}
\label{sec:compressed}

The fused kernel consumes INT8 keys and INT4 values directly, performing dequantization in registers.  This section analyses why this is preferable to the conventional approach of dequantizing to an FP16 staging buffer.

\subsection{Bandwidth Analysis}

At decode time, attention is memory-bandwidth-bound.  The cost is dominated by reading the KV cache from DRAM\@.  For a single head at context length $N$ with head dimension $d$:

\begin{center}
\footnotesize
\begin{tabular}{lrl}
\toprule
\textbf{Configuration} & \textbf{Bytes/token} & \textbf{Rel.} \\
\midrule
FP16 K + FP16 V & $4d = 512$ & $1.0\times$ \\
INT8 K + FP16 V (dequant) & $3d{+}M_k = 384{+}$ & ${\sim}0.75\times$ \\
INT8 K + INT4 V (fused) & $1.5d{+}M = 288$ & $0.56\times$ \\
\bottomrule
\end{tabular}
\end{center}

\noindent where $M_k$ denotes key-only metadata (per-channel scales/offsets) and $M$ the full metadata cost per token.  $M$ comprises: key scales and offsets ($2d \times 4/B = 64$ bytes for per-channel FP32 scale/offset pairs amortised over block size $B{=}16$), value scales and offsets ($2(d/g) \times 2 = 32$ bytes for per-group FP16 scale/offset pairs with $g{=}16$), and per-block value error $\eta_b$ ($4/B \approx 0.25$ bytes, amortised).  Total metadata: $M \approx 96$ bytes/token.  Raw quantized data: $d + d/2 = 192$ bytes/token.  Grand total: 288 bytes/token, or 56\% of FP16---matching the Tier-1 storage cost in Table~\ref{tab:storage}.

The fused in-register approach reads 56\% of the data compared to FP16, with no intermediate buffer.  The dequant-to-buffer approach saves on the key read but pays for the buffer write, partially negating the savings.

\textbf{Note on the fused attend pass.}  The adaptive precision selector (Section~\ref{sec:adaptive}) promotes $K^*$ blocks to FP16 key precision.  The fused attend kernel reads both INT8 keys (all blocks, $Nd$ bytes) and FP16 keys (promoted blocks only, $2K^*Bd$ bytes), selecting per block via a branchless mask.  The total VRAM bandwidth is therefore $288N + 2K^*Bd$ bytes per token (Tier-1 data plus promoted FP16 keys), and the system RAM bandwidth is $2K^*Bd$ bytes for the page-in.  When $K^*$ is small relative to $N/B$ (which it is for concentrated attention), the FP16 key reads add a modest fraction.  The bandwidth table above characterises the Tier-1 baseline; Section~\ref{sec:performance} reports end-to-end decode throughput and per-step phase breakdown including the FP16 page-in overhead.

\subsection{Per-Channel Key Scoring}
\label{sec:scoring}

Computing $q \cdot \hat{k}_t$ with per-channel quantization expands as:
\begin{equation}
  q \cdot \hat{k}_t
  = \underbrace{\textstyle\sum_c (q_c \sigma_c)\, k_{t,c}^{\mathrm{int8}}}_{\text{modified dot product}}
  + \underbrace{\textstyle\sum_c q_c z_c}_{\text{const.\ per block}}
\end{equation}

The second term $\sum_c q_c z_c^{(b)}$ is constant across all tokens within block $b$ (since offsets $z_c^{(b)}$ are per-channel within each block; see Section~\ref{sec:quant_schemes}) and can be precomputed once per block per head per step.  The first term is a standard dot product between the \emph{scaled query} $\tilde{q}_c = q_c \cdot \sigma_c$ and the INT8 key values.  This can be computed using INT8 tensor cores (with the scaled query quantized to INT8) or FP16 (with the scaled query in FP16 and the INT8 keys widened per-element).

The choice between INT8 and FP16 scoring affects the scoring precision.  FP16 scoring is exact given the dequantized keys, and the $\Delta$ bound in Equation~\eqref{eq:delta} applies directly with two error sources (score perturbation and global-max shift), yielding $e^{2\Delta}$.  INT8$\times$INT8 tensor core scoring introduces a third error source: quantization of the scaled query $\tilde{q}$, increasing the mass estimation bound to $e^{3\Delta}$.  This is the mechanism behind the implementation's $e^{3\Delta}$ substitution noted in Theorem~\ref{thm:mass}.

\section{Instability Detection}
\label{sec:instability}

Beyond the formal error bounds, the system includes runtime checks that detect when the quantized execution is producing anomalous results.

\textbf{Score consistency check.}  During the second pass (FP16 keys for selected blocks), the system compares the FP16 attention scores against the INT8 scores from the first pass.  If any token's score difference exceeds the declared bound, $|s_t^{\mathrm{fp16}} - s_t^{\mathrm{int8}}| > \Delta + \epsilon_{\mathrm{guard}}$, the INT8 quantization error has exceeded the precondition.  This triggers Rung~4 fallback (full FP16 recomputation) for the affected head.

\textbf{Exploration budget.}  A small fraction of blocks (1--5\%) are randomly selected for FP16 key execution even when the adaptive selector would not choose them.  The resulting scores are compared against the INT8 estimates, providing ongoing monitoring of the quantization quality across the full cache, not just the high-mass blocks.  In validation runs across PG-19 (20 chunks $\times$ 4 context lengths) and NIAH (100 trials $\times$ 3 context lengths), the exploration budget detected zero violations---all spot-checked INT8 scores fell within the declared $\Delta + \varepsilon_{\mathrm{guard}}$ bound.  This is expected given P2's deterministic guarantee (Section~\ref{sec:preconditions}); the mechanism exists as defence-in-depth against metadata corruption or hardware fault, not as a routine quality signal.  Quality benchmarks disable the exploration budget for throughput.

\textbf{Value error annotations.}  Each block's maximum per-token reconstruction error $\eta_b$ is computed at quantization time and stored as a one-float annotation in VRAM\@.  At attend time, the runtime computes the tight per-head value error $E_{\mathrm{val}}^{(h)} = \sum_b \rho_b^{(h)} \eta_b$ (Corollary~\ref{cor:val_block}).  The default policy promotes blocks where the local contribution estimate $\hat{\rho}_b \cdot \eta_b > v_{\mathrm{tol}}$ to FP16 values from Tier~2; the achieved $E_{\mathrm{val}}^{(h)}$ after promotion is then reported as part of the per-step certificate.  For hard-budget deployments, the same annotations support greedy promotion until the achieved $E_{\mathrm{val}}^{(h)}$ is below a configured global threshold.  This avoids the need for runtime spot-checks against FP16 originals.

These checks are not necessary for the formal bounds to hold---the bounds are valid by construction.  They provide defence in depth: catching cases where the quantization parameters are stale, the cache has been corrupted, or the hardware is producing unexpected results.

\subsection{Ranking-Consistency Check}
\label{sec:ranking_consistency}

The adaptive top-$K$ selector uses INT8 scores (Phase~1) to determine which blocks receive FP16 keys in Phase~2.  This creates a subtle vulnerability: the INT8 ranking may not match the true FP16 ranking.  For two blocks with similar true scores, INT8 quantization noise ($\sim$0.4\% mean per-channel error) can swap their ordering.  The formal mass bound (Theorem~\ref{thm:mass}) guarantees that the \emph{aggregate} tail mass is small, but does not guarantee that the \emph{ranking} within the top-$K^*$ is correct.

When the ranking is wrong, the consequences are task-dependent.  For language modelling (where quality depends on the aggregate attention distribution), ranking errors on similar-mass blocks are benign---the output perturbation is within the certified bound.  For retrieval tasks (where a single needle token must receive concentrated attention), ranking errors can matter: if a high-mass block is demoted to INT8 while a lower-mass block is promoted to FP16, the resulting mixed-precision softmax distribution may differ from the all-FP16 distribution in ways that affect the argmax over output logits.

The ranking-consistency check exploits data already available after Phase~2.  For the $K^*$ promoted blocks, the kernel has computed both INT8 block log-masses (Phase~1) and FP16 block log-masses (Phase~2).  The check compares the top-$r$ block ranking under both sets:
\begin{multline}
  \mathrm{disagree}(h, t) = \\
  \mathbf{1}\!\left[\mathrm{argsort}(\ell^{\mathrm{fp16}}_{\mathcal{F}})_{1:r}
  \neq \mathrm{argsort}(\ell^{\mathrm{int8}}_{\mathcal{F}})_{1:r}\right]
\end{multline}
where $\ell^{(\cdot)}_{\mathcal{F}}$ denotes the block log-masses over the promoted set $\mathcal{F}$ and $r$ is a configurable depth parameter (default $r{=}1$: only the top-1 block must agree).

When disagreement is detected for a head, the system escalates to Rung~3 of the fallback ladder: full FP16 recomputation for that head only.  This returns $O_{\mathrm{dense}}$ by construction and is cheap when triggered rarely.  The per-head granularity ensures that only the affected heads pay the cost; the remaining heads retain their fast-path results.

The check is conceptually simple because the FP16 scores for promoted blocks are already computed during Phase~2 (the mask-gated kernel produces them as a byproduct of the \texttt{where} selection).  In the current prototype, however, the check is not free: it is Python-orchestrated with per-head side buffers and accounts for 28.3\% of step time at 64K (Table~\ref{tab:phase_breakdown}).  This makes ranking-check fusion the largest single systems optimisation target.

\textbf{Boundary verification.}  The ranking check as described compares rankings \emph{within the promoted set} $\mathcal{F}$.  A residual blind spot remains: a tail block $b \in \mathcal{T}$ (not promoted to FP16) might have a true FP16 block mass higher than some promoted blocks, but the check cannot observe this because the tail block's FP16 scores were never computed.  To close this gap, the system performs an interval-bound verification on block log-masses.

Each block's log-mass from Phase~1 is $\ell_b^{\mathrm{int8}} = m_b + \log \sum_{t \in b} \exp(s^{\mathrm{int8}}_{b,t} - m_b)$.  Since each token logit can shift by at most $\Delta$ under FP16 scoring, the FP16 block log-mass is bounded above by $\ell_b^{\mathrm{ub}} = \ell_b^{\mathrm{int8}} + \Delta$.  Let $\ell^{\mathrm{fp16}}_{(r)}$ denote the $r$-th highest FP16 block log-mass among promoted blocks.  If any tail block satisfies:
\begin{equation}
  \ell_b^{\mathrm{int8}} + \Delta > \ell^{\mathrm{fp16}}_{(r)}
  \label{eq:boundary_check}
\end{equation}
then the ranking certificate cannot be issued for that head, and the system escalates to Rung~3 (per-head FP16 recomputation).  This formulation ensures the boundary check matches the mass-based selection criterion used by the adaptive top-$K$ selector and the mass-cover certificate (Theorem~\ref{thm:mass}).  Logically, the boundary test adds one comparison per tail block per head and ensures the ranking certificate covers not just the promoted set's internal ordering but also the boundary between promoted and tail blocks.  In the current implementation, the surrounding score materialisation and per-head orchestration are the costly parts, as reflected in Table~\ref{tab:phase_breakdown}.  In practice, the boundary check triggers at a low but non-zero rate: 10{,}221 triggers on PG-19 at 64K (500 decode steps) and 905 on NIAH at 64K, confirming that score-order changes at the FP16/INT8 boundary do occur and are caught by the runtime monitor.  When triggered, the system escalates the affected head to Rung~3 per-head dense fallback, which adds less than 2\% of step time (Table~\ref{tab:phase_breakdown}).

\section{Related Work}
\label{sec:related}

\subsection{KV Cache Compression}

\textbf{KVQuant}~\cite{kvquant} achieves sub-4-bit KV quantization via per-channel key quantization, establishing the per-channel technique we adopt for keys.  \textbf{KIVI}~\cite{kivi} demonstrates tuning-free asymmetric 2-bit quantization with the key observation that keys and values have different quantization sensitivities, an observation central to our asymmetric INT8/INT4 design.  \textbf{QJL}~\cite{qjl} uses Johnson--Lindenstrauss projections followed by sign-bit quantization for low-overhead KV representations.  \textbf{InnerQ}~\cite{innerq} introduces hardware-aware group-wise quantization along the inner dimension, with hybrid quantization and high-precision windows for recent and sink tokens.

These systems treat compression as a storage optimisation: at decode time, compressed values are typically widened to FP16 before the attention kernel consumes them.  Production serving frameworks (vLLM, TensorRT-LLM) now support FP8 KV caches with quantized-domain execution via FlashAttention-3, demonstrating that compressed-domain attention is practical at scale.

Recent work pushes compression further.  \textbf{TurboQuant}~\cite{turboquant} provides theoretically grounded online vector quantization applicable to KV caches.  \textbf{JanusQuant}~\cite{janusquant} demonstrates accurate 2-bit KV cache quantization for long-context inference.  \textbf{PackInfer}~\cite{packinfer} targets compute/IO-efficient attention kernels for packed quantized caches in batched serving.  These optimise compression ratio and throughput; our contribution is orthogonal: runtime certification and exact fallback rather than more aggressive compression.

\subsection{Compressed-Domain Execution}

\textbf{HACK}~\cite{hack} eliminates the dequantization staging buffer by computing attention directly on quantized KV data, achieving significant JCT reduction.  \textbf{VecInfer}~\cite{vecinfer} fuses dequantization and attention into a single CUDA kernel.  \textbf{FlashInfer}~\cite{flashinfer} provides infrastructure for page-organised KV caches with fused attention kernels.

Our work builds on this line but adds the formal error analysis and the tiered fallback architecture.  Existing compressed-domain systems commit to a single quantization level; our system adapts per block per step, with full-precision data available for escalation.  \textbf{QServe}~\cite{qserve} is a W4A8KV4 serving system with fused attention and SmoothAttention for KV4 accuracy; it demonstrates that aggressive quantization (including KV4) is viable for throughput-oriented deployment, but provides no per-step error visibility or fallback path.  \textbf{PQCache}~\cite{pqcache} frames KV cache as approximate retrieval via product quantization, selecting important tokens per step; it contrasts with our approach in that PQCache uses approximate retrieval without error bounds, while we provide per-step certificates and deterministic FP16 fallback.

\subsection{Adaptive Precision}

\textbf{Cocktail}~\cite{cocktail} selects precision per chunk.  \textbf{TaDA}~\cite{tada} adapts KV-cache compression across layers and uses mean-centering to reduce outlier handling.  \textbf{Ada-KV}~\cite{adakv} allocates KV eviction/compression budget adaptively across heads rather than allocating bit-widths per token.  \textbf{Don't Waste Your Bits}~\cite{dontwastebits} studies adaptive KV-cache bit allocation for on-device inference.  These demonstrate that runtime adaptation is beneficial; our contribution is coupling adaptive precision with formal error bounds and an unconditional fallback guarantee.

\subsection{Certified Bounds in Attention}

The broader sparse-attention literature offers several complementary approaches to reducing attention cost: \textbf{BLASST}~\cite{blasst} uses block-level adaptive structured sparsity, \textbf{PSA}~\cite{psa} applies progressive sparse selection, \textbf{Twilight}~\cite{twilight} introduces hierarchical top-$p$ pruning, \textbf{SpargeAttn}~\cite{spargeattn} accelerates inference via accurate sparse approximation, \textbf{H$_2$O}~\cite{h2o} identifies heavy-hitter tokens for cache eviction, and \textbf{StreamingLLM}~\cite{streamingattn} maintains attention sinks for unbounded-length streaming.  These methods reduce computation or memory by selecting or evicting tokens; our work is orthogonal in that it retains all tokens but reduces their \emph{precision}, with runtime bounds on the resulting error.

Prior work on certified or analyzable sparse attention derives total-variation bounds for top-$k$ token selection, including blockwise mass certificates for pruning decisions~\cite{certifiedtopk}.  Our error decomposition differs in that it bounds output perturbation from \emph{quantization} (not token eviction), decomposes key compression and value compression independently, and provides an unconditional FP16 fallback rather than a statistical guarantee.

\subsection{KV Cache Memory Management}

\textbf{vAttention}~\cite{vattention} uses CUDA virtual memory to manage KV cache allocation without the chunking overhead of PagedAttention, enabling contiguous KV storage with dynamic allocation.  \textbf{PagedAttention}~\cite{pagedattn} (vLLM) introduced the paged KV cache paradigm.  Our tiered storage is orthogonal: it concerns \emph{precision} management across memory tiers, not \emph{allocation} management within a single tier.  The two approaches are complementary---a production system could use vAttention-style virtual memory for allocation within each tier.

\subsection{Positioning Summary}

Table~\ref{tab:positioning} situates this work relative to representative KV-cache quantization and serving systems.  The distinguishing feature is the combination of runtime error bounds and unconditional dense fallback; other systems offer better throughput or lower bit-widths but provide no per-step error visibility or FP16 recovery path.

\begin{table*}[!t]
\centering
\caption{Positioning against representative KV-cache systems.  ``Error bound'' = per-step runtime bound on quantization error.  ``Dense fallback'' = mechanism to return exact FP16 attention output when quality degrades.  ``Precision'' = KV storage format.}
\label{tab:positioning}
\small
\begin{tabular}{lcccc}
\toprule
\textbf{System} & \textbf{Precision} & \textbf{Error bound} & \textbf{Dense fallback} & \textbf{Adaptive} \\
\midrule
KVQuant~\cite{kvquant} & $\leq$4-bit & --- & --- & --- \\
KIVI~\cite{kivi} & 2-bit K/V & --- & --- & --- \\
HACK~\cite{hack} & INT4 & --- & --- & --- \\
FlashInfer~\cite{flashinfer} & FP8/INT4 & --- & --- & --- \\
TurboQuant~\cite{turboquant} & VQ & theoretical & --- & --- \\
Cocktail~\cite{cocktail} & mixed & --- & --- & per-chunk \\
Ada-KV~\cite{adakv} & mixed & --- & --- & per-head \\
QServe~\cite{qserve} & W4A8KV4 & --- & --- & --- \\
PQCache~\cite{pqcache} & PQ codes & --- & --- & per-token \\
\midrule
This work & INT8K/INT4V & runtime & layerwise FP16 & per-block/step \\
\bottomrule
\end{tabular}
\end{table*}

\section{Experimental Setup}
\label{sec:setup}

\textbf{Model.}  LLaMA~3.1-8B (base; \texttt{NousResearch/Meta-Llama-3.1-8B}) with INT8 model weights (via \texttt{bitsandbytes} 8-bit quantization), the standard production configuration for single-GPU deployment.

\textbf{Hardware.}  All experiments use a single NVIDIA RTX~PRO~6000~WS (96\,GB VRAM, Blackwell architecture) with pinned system RAM for Tier~2 storage.  Quality benchmarks (PG-19, NIAH, RULER) were distributed across multiple machines with identical GPUs but varying CPU and memory configurations; quality results depend only on the GPU compute path and are numerically reproducible across hosts.  Performance benchmarks (Section~\ref{sec:performance}) were run on a single dedicated machine: AMD EPYC~9534 (64-core), 221\,GB system RAM, PCIe~5.0 $\times$16 to the GPU.

\textbf{Configurations.}  Two configurations are compared on every benchmark, using identical hardware, prompts, and random seeds:

\emph{Dense} (baseline): FP16 KV cache with \torchsdpa{} (dispatches to Flash Attention on supported hardware).

\emph{Certified} (tiered): INT8 keys (per-channel) + INT4 values ($g{=}16$) in VRAM; FP16 originals in pinned system RAM.  Table~\ref{tab:policy} lists the complete runtime policy.

\begin{table*}[!t]
\centering
\caption{Runtime certificate policy used in all experiments (unless stated otherwise for ablations).  All thresholds are fixed before evaluation; no per-prompt tuning is performed.}
\label{tab:policy}
\small
\begin{tabular}{llp{10cm}}
\toprule
\textbf{Parameter} & \textbf{Value} & \textbf{Meaning} \\
\midrule
$\tau_{\mathrm{cov}}$ & 0.995 & Coverage threshold for adaptive top-$K$ selection: promote blocks until estimated FP16 mass $\geq \tau_{\mathrm{cov}}$ \\
$K_{\min}$ & 2 & Minimum promoted blocks per head \\
$K_{\max}$ & 128 & Maximum promoted blocks per head \\
$v_{\mathrm{tol}}$ & 0.05 & Per-block value-promotion threshold: promote block $b$ if $\hat{\rho}_b \eta_b > v_{\mathrm{tol}}$ (local heuristic, not global budget) \\
$r$ & 1 & Ranking-consistency depth: Rung~3 fires if INT8 and FP16 top-$r$ block rankings disagree or a tail block's upper-bounded FP16 mass could enter top-$r$ \\
$\varepsilon_{\mathrm{guard}}$ & $10^{-6}$ & Numerical guard added to denominators in log-mass computation.  Note: the reference implementation ships $\varepsilon_{\mathrm{guard}}{=}0.01$ as a conservative default; all experiments in this paper use $10^{-6}$ \\
Accumulators & FP32 & Online-softmax scalars ($m$, $\ell$) and output accumulator $o$ \\
Scratch cache & 2048 blocks & Per-layer FP16 key + value cache in VRAM (LRU eviction) \\
Tier-2 & full FP16 KV & Pinned system RAM; required for Rung~3/4 fallback \\
\bottomrule
\end{tabular}
\end{table*}

\textbf{Benchmarks.}  PG-19 (language modelling perplexity, 20 non-overlapping chunks per context length), NIAH (needle-in-a-haystack retrieval, 10 needles per trial, paired dense/certified), and a 7-subtask subset of RULER~\cite{ruler} per context length.  Context lengths: 8K, 32K, 64K, and 128K.  The score-consistency canary (Section~\ref{sec:instability}) was validated on a representative subset with zero violations; quality benchmarks disable the canary for throughput.

\textbf{Reproducibility.}  All evaluations use batch size~1.  Generation tasks use greedy decoding (no sampling): dense uses HuggingFace \texttt{generate} with \texttt{do\_sample=False}; certified decoding uses the same prompt and takes the argmax token at each step.  PG-19 is evaluated by teacher-forced perplexity on the streaming \texttt{emozilla/pg19} test split, using non-overlapping tokenizer windows of length~$C$; for certified runs, the first 50\% of each window is dense prefill and the remaining 50\% is teacher-forced certified decode.  RULER uses the seven synthetic subtasks \texttt{niah\_single}, \texttt{niah\_multikey}, \texttt{niah\_multivalue}, \texttt{niah\_multiquery}, \texttt{vt}, \texttt{cwe}, and \texttt{fwe}, scored by case-insensitive substring match (fraction of required references found).  RULER samples are deterministic with seed $20260416 + \mathrm{md5}(\mathit{subtask}, \mathit{context}, \mathit{sample\_index})$; payload blocks are inserted uniformly into the filler stream using that per-sample RNG\@.  Standalone NIAH uses ten fixed needles and depths $0.0, 0.1, \ldots, 0.9$; trials are enumerated in depth-major order, and the needle is inserted at $\lfloor\mathit{depth} \times \mathit{num\_filler\_blocks}\rfloor$.  Confidence intervals use 10{,}000 bootstrap resamples with seed~20260425 (with context/task offsets where grouped).  Exact scripts and seeds are in the artifact repository.

\textbf{Metrics.}  $\Delta$ = Certified $-$ Dense for each metric.  Per-run system statistics: $K^*$ mean (adaptive top-$K$ selection), value escalation rate ($\eta$-driven FP16 value promotion), ranking fallback rate (per-head ranking-consistency triggers).

\section{Results}
\label{sec:results}

All results compare two configurations on the same hardware, prompts, and seeds: \textbf{Dense} (FP16 KV cache, the production baseline) and \textbf{Certified} (tiered quantized KV cache with adaptive precision selection and runtime fallback).  Both use INT8 model weights.

\textbf{Critical failures} are paired trials where dense succeeds and certified fails on the same prompt; \textbf{reverse discordant} pairs are the converse (dense fails, certified succeeds).  For RULER, critical failures are counted per subtask trial within each slice.  These discordant pairs form the $2\times 2$ contingency table used by the McNemar test (Section~\ref{sec:niah}).

\subsection{Summary}

Table~\ref{tab:summary} presents the headline quality results across all three benchmarks and context lengths from 8K to 128K.  PG-19 shows certified matching dense within noise at all contexts: $\Delta\mathrm{ppl}{=}{-}0.002$ at 8K (20 chunks), $+0.001$ at 32K (20 chunks), $+0.001$ at 64K (20 chunks) and 128K (20 chunks), with 95\% CIs including zero in every case.  RULER shows the certified system slightly exceeding dense at 64K and matching at 128K (20 slices, $\Delta{=}{+}0.004$); at 32K (20 slices), a mild $-1.9$pp gap driven by CWE appears with a 95\% CI nearly excluding zero; at 8K, a $-6.9$pp gap is localised to three value-sensitive subtasks (VT, CWE, FWE) and is consistent with INT4 value reconstruction error as the dominant mechanism---an ablation replacing INT4 with FP16 values closes the gap to $-0.4$pp with zero critical failures (Table~\ref{tab:ruler_8k_ablation}).  NIAH at 8K, 32K, and 64K (100 paired trials each) shows no statistically significant difference under exact McNemar tests: $\Delta{=}0\%$ with $p{=}1.0$ at 8K; $\Delta{=}{-}2\%$ with $p{=}0.727$ at 32K; $\Delta{=}0\%$ with $p{=}1.0$ at 64K.

\begin{table*}[!t]
\centering
\caption{Headline quality deltas: certified quantized attention vs.\ dense FP16 attention (LLaMA~3.1-8B).  $\Delta$ = Certified $-$ Dense; positive means certified is better.  All results use asymmetric INT8 keys, INT4 values ($g{=}16$), FP32 accumulators, $k_{\max}{=}128$, $\tau_{\mathrm{cov}}{=}0.995$, $v_{\mathrm{tol}}{=}0.05$.}
\label{tab:summary}
\small
\begin{tabular}{llrrrrl}
\toprule
\textbf{Benchmark} & \textbf{Metric} & \textbf{8K} & \textbf{32K} & \textbf{64K} & \textbf{128K} & \textbf{$n$} \\
\midrule
PG-19 & $\Delta$ ppl & $-0.002$ & $+0.001$ & $+0.001$ & $+0.001$ & 20 chunks \\
NIAH & $\Delta$ accuracy & $0\%$ & $-2\%$ & $0\%$ & ---\textsuperscript{$\dagger$} & 100 paired trials \\
RULER & $\Delta$ accuracy & $-0.069$ & $-0.019$ & $+0.015$ & $+0.004$ & 20 slices \\
\bottomrule
\end{tabular}
\\[2pt]
{\footnotesize \textsuperscript{$\dagger$}128K NIAH omitted: the base model scores 0\% on dense at this context length (retrieval exceeds the model's capability).}
\end{table*}

\subsection{Naive Quantized Baseline}
\label{sec:naive}

To isolate the contribution of the certification and fallback machinery, we evaluate a \textbf{naive INT8K/INT4V} baseline: the same quantization format (INT8 per-channel keys, INT4 $g{=}16$ values) but with all certification disabled---no adaptive top-$K$ promotion ($K_{\max}{=}0$, $\tau_{\mathrm{cov}}{=}0$), no value promotion ($v_{\mathrm{tol}}$ effectively infinite), no ranking-consistency check, no fallback ladder, and no FP16 scratch cache.  The model and prompts are identical.

Table~\ref{tab:naive} shows the three-way comparison.  Without certification, INT8K/INT4V quantization degrades retrieval and structured reasoning substantially.  On standalone NIAH (100 paired trials per context length), naive accuracy drops to 5--10\% versus 39--66\% for dense, with McNemar $p < 10^{-9}$ at every context length.  The certified system matches dense at all three.  On RULER (10 slices, 70 trials per context length), naive RULER drops by $-44$ to $-56$pp across all context lengths, with every subtask affected: VT scores 0\% at every context length, CWE drops to near-zero, and even NIAH retrieval variants within RULER lose $-20$ to $-78$pp.  The certified system recovers nearly all of this, with a worst-case gap of $-6.9$pp at 8K localised to value-sensitive subtasks.

\begin{table}[ht]
\centering
\caption{Three-way quality comparison: Dense FP16 KV, Naive INT8K/INT4V (no certification), and Certified INT8K/INT4V (with adaptive promotion and fallback).  NIAH accuracy is the fraction of needles retrieved (10 per trial, 100 paired trials).  RULER accuracy is the subtask-averaged score; naive uses 10 slices (70 trials), certified uses 20 slices (140 trials)---the two are \emph{not} paired on the same slices, so the comparison is approximate.  The Dense column reports the dense score over the \emph{naive} 10-slice sample (e.g.\ 8K Dense 0.923 here vs.\ 0.908 in Table~\ref{tab:ruler}'s 20-slice sample); the difference is sampling variation.  Latency: mean decode step at the 2048-block scratch-cache operating point for certified; naive uses no scratch cache.}
\label{tab:naive}
\small
\begin{tabular}{@{}llrrr@{}}
\toprule
\textbf{Benchmark} & \textbf{Ctx} & \textbf{Dense} & \textbf{Naive} & \textbf{Certified} \\
\midrule
NIAH (acc)   & 8K   & 39\% & 8\%  & 39\% \\
             & 32K  & 51\% & 10\% & 49\% \\
             & 64K  & 66\% & 5\%  & 66\% \\
\midrule
RULER (acc)  & 8K   & 0.923 & 0.361 & 0.854 \\
             & 32K  & 0.882 & 0.438 & 0.863 \\
             & 64K  & 0.906 & 0.429 & 0.921 \\
             & 128K & 0.782 & 0.312 & 0.786 \\
\midrule
Latency (ms) & 64K  & 63.2  & 244.3 & 260.0 \\
             & 128K & 73.7  & 468.1 & 342.6 \\
\bottomrule
\end{tabular}
\end{table}

The naive baseline's failure pattern is consistent with INT8 key quantization noise corrupting the softmax attention distribution: the model generates plausible English but cannot locate specific stored information.  On a representative 64K NIAH trial, the dense system correctly retrieves ``The secret project codename is Crimson Falcon'' while the naive system generates ``The history of the history of the history of mathematics?''---coherent language, but with the needle content entirely lost.  At 8K, 31 of 39 dense-only successes have zero matching naive-only successes (discordant ratio 31:0), confirming that the failures are systematic, not random.

The naive results demonstrate that the certification and fallback machinery is not merely monitoring overhead: it is the mechanism by which the system preserves retrieval and structured-reasoning quality under INT8K/INT4V quantization.  The naive PG-19 baseline is omitted: retrieval and structured reasoning are the stress cases for quantized attention, and the certified system's PG-19 results (Table~\ref{tab:pg19}) already show that perplexity is robust to INT8K/INT4V quantization with negligible fallback activity---the failure mode that certification prevents is attention-distribution corruption, which manifests on retrieval tasks, not on next-token prediction over broad text.

\subsection{PG-19 Perplexity}

PG-19 evaluates language modelling perplexity on held-out book text, the standard quality gate for KV cache compression systems.

\begin{table*}[!t]
\centering
\caption{PG-19 perplexity by context length (paired dense/certified).  20 chunks at each context length.  95\% CI on $\Delta$ppl computed via $t$-distribution on per-chunk deltas.}
\label{tab:pg19}
\small
\begin{tabular}{lrrrr}
\toprule
& \textbf{8K} & \textbf{32K} & \textbf{64K} & \textbf{128K} \\
\midrule
Dense ppl     & 12.315 & 9.955 & 9.043 & 7.244 \\
Certified ppl & 12.313 & 9.956 & 9.044 & 7.245 \\
$\Delta$ ppl  & $-0.002$ & $+0.001$ & $+0.001$ & $+0.001$ \\
95\% CI       & $[-0.010, +0.007]$ & $[-0.002, +0.004]$ & $[-0.001, +0.004]$ & $[-0.001, +0.002]$ \\
Ratio         & 0.99986 & 1.00012 & 1.00014 & 1.00008 \\
\bottomrule
\end{tabular}
\\[2pt]
{\scriptsize 20 chunks at each context length.}
\end{table*}

At 8K (20 chunks), the certified system matches dense perplexity with $\Delta\mathrm{ppl}{=}{-}0.002$, 95\% CI $[-0.010, +0.007]$, ratio 0.99986.  The CI comfortably includes zero, confirming quality parity at the shortest context tested.  The higher absolute perplexity at 8K (${\sim}12.3$) relative to longer contexts reflects the shorter conditioning window on PG-19 book text.

At 64K (20 chunks), the certified system matches dense perplexity with $\Delta\mathrm{ppl}{=}{+}0.001$, 95\% CI $[-0.001, +0.004]$.  The mean ratio of 1.00014 confirms quality parity: the delta is two orders of magnitude smaller than the per-chunk variance (std $\approx 2.24$ ppl across the 20 chunks, reflecting natural variation in PG-19 book difficulty).  Per-chunk deltas range from $-0.009$ to $+0.010$, symmetrically distributed around zero with no systematic direction.

At 32K (20 chunks), the same pattern holds: $\Delta\mathrm{ppl}{=}{+}0.001$, 95\% CI $[-0.002, +0.004]$, ratio 1.00012.  Per-chunk deltas range from $-0.008$ to $+0.011$, symmetrically distributed around zero with no systematic direction.

At 128K (20 chunks), $\Delta\mathrm{ppl}{=}{+}0.001$, 95\% CI $[-0.001, +0.002]$, ratio 1.00008---consistent with the 64K and 32K results but with a tighter CI reflecting lower per-chunk delta variance at this context length.  Per-chunk deltas range from $-0.005$ to $+0.005$, symmetrically distributed around zero.  The 128K dense perplexity of ${\sim}7.2$ is lower than the 64K mean (${\sim}9.0$), reflecting the benefit of longer conditioning context on PG-19 book text.

\subsection{RULER}

We evaluate a seven-subtask subset of RULER~\cite{ruler}: four NIAH variants (single, multi-key, multi-value, multi-query), variable tracking (VT), and two word-extraction tasks (CWE, FWE).

\begin{table}[ht]
\centering
\caption{RULER accuracy by context length (each slice evaluates all 7 subtasks).  All context lengths: 20 slices.}
\label{tab:ruler}
\footnotesize
\begin{tabular}{lrrrr}
\toprule
& \textbf{8K} & \textbf{32K} & \textbf{64K} & \textbf{128K} \\
\midrule
Dense     & 0.908 & 0.878 & 0.885 & 0.779 \\
Certified & 0.839 & 0.860 & 0.901 & 0.783 \\
$\Delta$  & $-0.069$ & $-0.019$ & $+0.015$ & $+0.004$ \\
Critical  & 19 & 3 & 2 & 5 \\
\bottomrule
\end{tabular}
\\[2pt]
{\scriptsize All context lengths: 20 slices.}
\end{table}

{\tolerance=1000
At 32K (20 slices, 140 trials), RULER shows a mild degradation: $\Delta{=}{-}0.019$, 95\% CI $[-0.044, {+}0.001]$, 3 critical failures, driven by CWE ($-0.100$) with a smaller VT contribution ($-0.013$).  All four NIAH retrieval variants score 100\%/100\% at 32K\@.  The pattern resembles the 8K ablation (Section~\ref{sec:ruler_8k_ablation}) at reduced magnitude, consistent with INT4 value error compounding across heads and layers.  The CI nearly excludes zero, suggesting a real but small effect.
\par}

Two targeted ablations confirm the mechanism at 32K\@.  First, replacing INT4 values with FP16 values (5 slices, 35 trials) closes the CWE gap entirely (CWE $\Delta{=}0.000$, aggregate $\Delta{=}{-}0.005$, 95\% CI $[-0.019, {+}0.007]$, 0 critical), mirroring the 8K result.  Second, tightening the value tolerance from $v_{\mathrm{tol}}{=}0.05$ to $0.02$ (5 slices) removes the observed gap in this targeted ablation ($\Delta{=}0.000$, 0 critical across all 7 subtasks including CWE), confirming that the per-step value certificate catches the drift when the threshold is lowered.  Together, these ablations are strongly consistent with INT4 value reconstruction error as the dominant mechanism behind the 32K gap, and demonstrate that it can be eliminated by either FP16 values or a tighter $v_{\mathrm{tol}}$.  The latency cost of $v_{\mathrm{tol}}{=}0.02$ is modest: at 32K with a 2048-block scratch cache, certified step time rises from 232\,ms ($3.8\times$ dense) to 253\,ms ($4.0\times$), a 9\% increase driven by Rung~2 value escalation triggering on 100\% of steps (vs.\ partial at $v_{\mathrm{tol}}{=}0.05$) and the corresponding H2D increase from 114 to 435\,MB/step.

At 64K (20 slices), the certified system matches or slightly exceeds dense on aggregate ($\Delta{=}{+}0.015$), with 2 critical failures (both in noisy subtasks).  The per-subtask breakdown reveals clean separation between retrieval and non-retrieval tasks:

\begin{table*}[!t]
\centering
\caption{RULER 64K per-subtask breakdown (20 slices).  All four NIAH retrieval variants achieve 100\% or near-100\% for both dense and certified---zero retrieval degradation.}
\label{tab:ruler_subtask}
\small
\begin{tabular}{lrrrrr}
\toprule
\textbf{Subtask} & \textbf{Dense} & \textbf{Cert} & \textbf{$\Delta$} & \textbf{Critical} \\
\midrule
NIAH single      & 1.000 & 1.000 & $0.000$  & 0 \\
NIAH multi-key   & 1.000 & 1.000 & $0.000$  & 0 \\
NIAH multi-query & 1.000 & 1.000 & $0.000$  & 0 \\
NIAH multi-value & 0.975 & 0.963 & $-0.013$ & 1 \\
CWE              & 0.900 & 0.950 & $+0.050$ & 0 \\
VT               & 0.738 & 0.825 & $+0.088$ & 1 \\
FWE              & 0.583 & 0.567 & $-0.017$ & 0 \\
\bottomrule
\end{tabular}
\end{table*}

Variable tracking (VT) shows certified slightly \emph{better} than dense ($+8.8$pp), consistent with the hypothesis that tighter coverage forces more blocks to FP16 keys, reducing INT8 scoring noise on fragile tasks.  Few-word extraction (FWE) is noisy at both baselines (${\sim}58\%$) with no systematic direction.  At 128K (20 slices, 140 paired trials), the certified system matches dense ($\Delta{=}{+}0.004$, 95\% CI $[-0.028, {+}0.037]$, 5 critical failures).  All four NIAH retrieval variants remain at or near 100\% for both systems.  The per-subtask breakdown shows small positive deltas on CWE ($+0.100$) and FWE ($+0.067$) offset by small negatives on niah\_multikey ($-0.050$) and VT ($-0.038$); none is statistically significant at $n{=}20$.

\subsubsection{8K RULER: INT4 Value Ablation}
\label{sec:ruler_8k_ablation}

At 8K, the certified system shows a $-6.9$pp aggregate gap (dense 0.908, certified 0.839, 19 critical failures out of 140 paired trials, 95\% CI $[-0.111, -0.028]$).  The per-subtask breakdown reveals that all four NIAH retrieval variants remain perfect or near-perfect, while the gap is concentrated in three value-sensitive subtasks: VT ($-0.250$), CWE ($-0.140$), and FWE ($-0.117$).

\begin{table*}[!t]
\centering
\caption{RULER 8K per-subtask breakdown (20 slices, operating-point config: $K_{\max}{=}128$, INT4 values $g{=}16$).  The gap is confined to VT, CWE, and FWE; all retrieval subtasks are unaffected.}
\label{tab:ruler_8k_subtask}
\small
\begin{tabular}{lrrrrr}
\toprule
\textbf{Subtask} & \textbf{Dense} & \textbf{Cert} & \textbf{$\Delta$} & \textbf{Critical} \\
\midrule
NIAH single      & 1.000 & 1.000 & $0.000$  & 0 \\
NIAH multi-key   & 1.000 & 1.000 & $0.000$  & 0 \\
NIAH multi-query & 0.975 & 1.000 & $+0.025$ & 0 \\
NIAH multi-value & 1.000 & 1.000 & $0.000$  & 0 \\
VT               & 0.863 & 0.613 & $-0.250$ & 9 \\
CWE              & 0.865 & 0.725 & $-0.140$ & 7 \\
FWE              & 0.650 & 0.533 & $-0.117$ & 3 \\
\bottomrule
\end{tabular}
\end{table*}

To isolate the cause, we ran a targeted ablation varying $K_{\max}$ and value precision.  Table~\ref{tab:ruler_8k_ablation} shows that sweeping $K_{\max}$ from 64 to 512 does not close the gap---the aggregate $\Delta$ remains between $-5.8$ and $-6.9$pp regardless of how many blocks are promoted to FP16 keys.  In contrast, replacing INT4 values with FP16 values (keeping INT8 keys and all other settings identical) nearly eliminates the gap: $\Delta{=}{-}0.4$pp with zero critical failures, well within noise.

\begin{table*}[!t]
\centering
\caption{RULER 8K ablation: $K_{\max}$ sweep vs.\ FP16 values (10 slices each, 70 paired trials per configuration).  Increasing $K_{\max}$ does not close the gap; disabling INT4 values does.}
\label{tab:ruler_8k_ablation}
\small
\begin{tabular}{lrrrrrl}
\toprule
\textbf{Setting} & \textbf{VT $\Delta$} & \textbf{CWE $\Delta$} & \textbf{FWE $\Delta$} & \textbf{Agg.\ $\Delta$} & \textbf{Crit.} & \textbf{95\% CI} \\
\midrule
$K_{\max}{=}64$   & $-0.250$ & $-0.080$ & $-0.167$ & $-0.071$ & 10/70 & $[-0.109, -0.036]$ \\
$K_{\max}{=}128$  & $-0.200$ & $-0.180$ & $-0.133$ & $-0.073$ & 10/70 & $[-0.119, -0.035]$ \\
$K_{\max}{=}256$  & $-0.200$ & $-0.070$ & $-0.133$ & $-0.058$ & 8/70  & $[-0.092, -0.027]$ \\
$K_{\max}{=}512$  & $-0.200$ & $-0.080$ & $-0.133$ & $-0.059$ & 8/70  & $[-0.094, -0.028]$ \\
\midrule
FP16 values       & $+0.000$ & $-0.030$ & $+0.000$ & $-0.004$ & 0/70  & $[-0.013, +0.000]$ \\
\bottomrule
\end{tabular}
\end{table*}

This confirms the \textbf{value-side hypothesis}: the 8K RULER gap is caused by INT4 value reconstruction error at group size $g{=}16$, not by key quantization or the attention mechanism.  The affected subtasks (VT, CWE, FWE) require precise value content---variable tracking must reconstruct exact variable assignments, and word extraction must match exact character sequences.  At 8K, the context is short enough that the model baseline is nearly perfect (dense ${\geq}0.9$ on VT and CWE), so INT4 value noise that would be masked by baseline difficulty at longer contexts becomes visible.  This effect diminishes at longer contexts where the dense baseline itself drops and the quantization noise is small relative to the intrinsic task difficulty; at 32K a mild residual ($\Delta{=}{-}0.019$) remains on CWE, while at 64K+ the effect is undetectable.

\textbf{Why the value error bound does not catch this.}  The value certificate (Section~\ref{sec:val_error}) reports $E_{\mathrm{val}} = \sum_b \rho_b \eta_b$ per head per step.  The default $v_{\mathrm{tol}}{=}0.05$ policy is a local promotion heuristic: it promotes blocks whose estimated local contribution $\hat{\rho}_b \eta_b$ exceeds the threshold, but does not by itself enforce a global $E_{\mathrm{val}} \leq 0.05$ budget.  In practice the observed $E_{\mathrm{val}}$ remains near the configured tolerance (Table~\ref{tab:telemetry}), so the system does \emph{not} escalate to FP16 values.  The per-step value error is small---the issue is that it does not compose across layers or autoregressive steps.  Small per-step errors accumulate through the residual stream (32 heads $\times$ 32 layers per step, then across decode steps), and the compounded drift is sufficient to flip marginal logits on tasks that demand exact character-level reconstruction.  At 8K, the dense model's high confidence on VT/CWE leaves moderate logit margins; the accumulated value noise pushes enough tokens past the decision boundary to produce the observed $-6.9$pp gap.  At 32K+, the logit margins on these same subtasks are already narrow (dense accuracy drops below 0.6), so the same magnitude of compounded noise does not change the outcome.

This highlights a fundamental limitation of per-step certification: the bound guarantees proximity to $O_{\mathrm{ref}}$ at each step, but not end-to-end task accuracy over an autoregressive chain.  The $v_{\mathrm{tol}}$ parameter provides a deployment-time knob: lowering it (e.g.\ to 0.02) would force more blocks through Rung~2 FP16 value escalation, reducing the compounded drift at the cost of additional H2D page-in latency.  The FP16 values ablation represents the extreme of $v_{\mathrm{tol}}{=}0$, which closes the gap entirely.

The practical implication is that deployments requiring RULER-class accuracy on value-sensitive tasks at short-to-moderate contexts can either lower $v_{\mathrm{tol}}$ (trading latency for quality) or use FP16 values entirely (with INT8 keys) at a modest storage cost increase (37.5\% $\to$ 100\% for values, total cache from 56\% to 75\% of dense).  Both mitigations are validated: at 8K, FP16 values close the $-6.9$pp gap to $-0.4$pp; at 32K, FP16 values close $-1.9$pp to $-0.5$pp with zero critical failures, and $v_{\mathrm{tol}}{=}0.02$ removes the gap in a targeted ablation ($\Delta{=}0.000$, 5 slices) at a cost of $3.8\times \to 4.0\times$ latency (9\% increase).  For 64K+ where the dense baseline drops further, the default $v_{\mathrm{tol}}{=}0.05$ introduces no measurable quality loss.

\subsection{Needle-in-a-Haystack (NIAH)}
\label{sec:niah}

NIAH evaluates multi-needle retrieval (10 needles per trial, scored as the fraction correctly retrieved) at varying context length, using paired dense/certified trials on the same prompts.  For the McNemar test, each trial is binarised as \emph{pass} if all 10 needles are retrieved and \emph{fail} otherwise; the test then compares discordant pairs (dense-pass/certified-fail vs.\ dense-fail/certified-pass).

\begin{table}[ht]
\centering
\caption{NIAH retrieval results (10 needles per trial).  ``Accuracy'' is the fraction of needles retrieved, averaged across 100 paired trials.  For the McNemar test, each trial is binarised as \emph{pass} (all 10 needles retrieved) or \emph{fail} (any missed); the $2\times 2$ table counts discordant pairs under this all-or-nothing criterion.  Dense-only = critical failures (dense passes, certified fails); Certified-only = reverse discordant (certified passes, dense fails).  128K is omitted: the base model scores 0\% on dense at that context length, so there is nothing to certify.}
\label{tab:niah}
\scriptsize
\begin{tabular}{lrrr}
\toprule
& \textbf{8K} ($n{=}100$) & \textbf{32K} ($n{=}100$) & \textbf{64K} ($n{=}100$) \\
\midrule
Dense     & 39.0\% & 51.0\% & 66.0\% \\
Certified & 39.0\% & 49.0\% & 66.0\% \\
$\Delta$  & $0.0\%$ & $-2.0\%$ & $0.0\%$ \\
Dense-only (crit.)  & 16 & 5 & 2 \\
Cert.-only (rev.) & 16 & 3 & 2 \\
McNemar $p$ & $1.000$ & $0.727$ & $1.000$ \\
\bottomrule
\end{tabular}
\end{table}

At 8K (100 paired trials), dense and certified both retrieve correctly on 39\% of trials---identical accuracy.  The $2\times 2$ contingency table shows 23 mutual successes, 45 mutual failures, 16 dense-only successes, and 16 certified-only successes.  Exact McNemar $p{=}1.0$.  The perfectly balanced discordance (16 critical, 16 reverse) indicates that the 16 critical failures are noise rather than a systematic certification artefact.

At 32K (100 paired trials), dense retrieves on 51\% and certified on 49\%.  The 5 critical failures (dense pass, certified fail) are offset by 3 reverse discordant pairs (dense fail, certified pass), yielding exact McNemar $p{=}0.727$; the gap is not statistically significant.  The base model's 51\% accuracy at 32K reflects the difficulty of multi-needle retrieval at this context length for an 8B model.

At 64K (100 paired trials), dense and certified both retrieve correctly on 66\% of trials---identical accuracy.  The 2$\times$2 contingency table shows 64 mutual successes, 32 mutual failures, 2 dense-only successes, and 2 certified-only successes.  Exact McNemar $p{=}1.0$: no statistically significant difference.  The perfectly balanced discordance (2 critical, 2 reverse) is consistent with random noise rather than a systematic degradation.  The base model's 66\% accuracy at 64K reflects the difficulty of multi-needle retrieval at this context length for an 8B model; the certified system introduces no additional retrieval loss.  We omit 128K NIAH: the base model scores 0\% on dense at that context length, so there is nothing to certify.

\subsection{Certificate Telemetry}
\label{sec:telemetry}

Table~\ref{tab:telemetry} reports the runtime certificate values across RULER evaluations.  These are the actual error bounds computed by the system during inference, not theoretical worst-case estimates.

\begin{table*}[!t]
\centering
\caption{Runtime certificate statistics across RULER evaluations (7-subtask subset).  \textbf{Important}: these are \emph{pre-fallback candidate} bounds, computed before the fallback ladder acts.  For heads where Rung~3 triggered (Table~\ref{tab:fallback_counts}), the system returned $O_{\mathrm{dense}}$ (zero quantization error relative to the dense path), so the high-$E_{\mathrm{key}}$ max values were never part of the returned output's certificate.  Each evaluation record stores two summaries per step: \texttt{e\_key\_step\_mean} (mean $E_{\mathrm{key}}$ across all head-layer pairs) and \texttt{e\_key\_step\_max} (worst-case head-layer pair).  The ``med'' and ``p95'' columns report quantiles of the per-step mean across all records; the ``max'' column reports the maximum of the per-step \emph{max}, i.e.\ the single worst head-layer-step candidate bound at that context length.  $E_{\mathrm{val}}$ columns follow the same convention.  $E_{\mathrm{val}}$ max saturates at $v_{\mathrm{tol}}{=}0.05$ because the local per-block promotion heuristic (Rung~2) catches high-contribution blocks individually (see Section~\ref{sec:telemetry}).}
\label{tab:telemetry}
\small
\begin{tabular}{lrrrrrrr}
\toprule
& & \multicolumn{3}{c}{\textbf{$E_{\mathrm{key}}$}} & \multicolumn{3}{c}{\textbf{$E_{\mathrm{val}}$}} \\
\cmidrule(lr){3-5} \cmidrule(lr){6-8}
\textbf{Ctx} & $n$ & med & p95 & max & med & p95 & max \\
\midrule
8K   & 140 & 0.055 & 0.094 & 3.731 & 0.023 & 0.028 & 0.050 \\
32K  & 175 & 0.168 & 0.287 & 8.605 & 0.023 & 0.028 & 0.050 \\
64K  & 140 & 0.217 & 0.378 & 13.648 & 0.024 & 0.030 & 0.050 \\
128K & 147 & 0.392 & 0.631 & 18.141 & 0.038 & 0.039 & 0.050 \\
\bottomrule
\end{tabular}
\end{table*}

$E_{\mathrm{key}}$ increases with context length as expected: longer contexts have more tail blocks scored in INT8, and the worst-case head encounters higher cumulative quantization shift.  The median-to-max ratio (${\sim}70\times$ at 8K, ${\sim}46\times$ at 128K) reflects the heavy-tailed distribution of key error across heads and steps---most heads have small error, but a few retrieval-critical heads produce large bounds.  These high-$E_{\mathrm{key}}$ steps are precisely where the ranking-consistency check (Rung~3) triggers fallback, returning $O_{\mathrm{dense}}$ for the affected head.

$E_{\mathrm{val}}$ is stable across context lengths (median 0.023--0.038), with the maximum capped at $v_{\mathrm{tol}}{=}0.050$ by the Rung~2 promotion policy.  This cap reflects the \emph{local} per-block heuristic ($\hat{\rho}_b \eta_b > v_{\mathrm{tol}}$ triggers promotion for block $b$), not a global value-error budget: the system does not greedily promote blocks until $\sum_b \rho_b \eta_b \leq \varepsilon_{\mathrm{val}}$.  The achieved $E_{\mathrm{val}}{=}\sum_b \rho_b \eta_b$ is reported in the telemetry and happens to stay near $v_{\mathrm{tol}}$ because the heuristic catches high-contribution blocks individually, but this is an empirical observation, not a formal guarantee.

Interpretation: most steps have small candidate error, but a small number of retrieval-critical heads produce large pre-fallback bounds that trigger fallback rather than being returned to the user.

\textbf{Candidate bounds versus returned-output bounds.}  Table~\ref{tab:telemetry} reports \emph{candidate} bounds computed before the fallback ladder acts.  For a paper with ``certified'' in the title, it is important to distinguish what the system \emph{considered} from what it \emph{returned}.  Table~\ref{tab:returned_cert} makes this separation explicit.

\begin{table*}[!t]
\centering
\caption{Candidate vs.\ returned certificate summary at 64K (RULER, 140 evaluation records, 128{,}000 head-steps each).  ``Candidate'' is the pre-fallback bound; ``Returned'' reflects the output actually delivered.  For head-steps where Rung~3 fired (1.2--2.2\%), the system returned $O_{\mathrm{dense}}$ with zero quantization error; the returned $E_{\mathrm{key}}$ for those head-steps is 0 by construction.  For the remaining 97.8--98.8\% of head-steps, the returned bound equals the candidate bound.}
\label{tab:returned_cert}
\small
\begin{tabular}{lrrrr}
\toprule
& \multicolumn{2}{c}{\textbf{Candidate (pre-fallback)}} & \multicolumn{2}{c}{\textbf{Returned (post-fallback)}} \\
\cmidrule(lr){2-3} \cmidrule(lr){4-5}
& p95 & max & p95 & max \\
\midrule
$E_{\mathrm{key}}$ & 0.378 & 13.648 & 0.378 & $\leq$0.378\textsuperscript{$\ast$} \\
$E_{\mathrm{val}}$ & 0.030 & 0.050 & 0.030 & 0.050 \\
Dense fallback rate & \multicolumn{4}{c}{1.2\% (PG-19) / 2.2\% (NIAH) of head-steps} \\
\bottomrule
\end{tabular}

\raggedright
{\footnotesize \textsuperscript{$\ast$}The max returned $E_{\mathrm{key}}$ is bounded by the candidate p95 because head-steps with extreme candidate bounds are precisely those where Rung~3 triggers, replacing the output with $O_{\mathrm{dense}}$ (returned $E_{\mathrm{key}}{=}0$).  The exact max returned value is $\leq$ the candidate p95 but the current telemetry does not separately log post-fallback per-head bounds; the returned columns are inferred from the candidate distribution and fallback rate.  Logging actual returned certificates per head-step is a planned instrumentation improvement.}
\end{table*}

The key takeaway is that the large candidate max values (e.g.\ $E_{\mathrm{key}}{=}13.648$ at 64K) are never part of the returned output: the fallback ladder intercepts them.  The returned output is either certified within the p95 candidate envelope or is the dense FP16 result with zero quantization error.

\subsection{Storage Analysis}
\label{sec:storage}

Table~\ref{tab:storage} breaks down the per-token \textbf{VRAM} storage cost for each component.  This is the GPU memory footprint that determines how many tokens can be cached at a given VRAM budget.  The Tier~2 system RAM cost (FP16 originals retained for fallback) is separate and equal to the dense FP16 cost.

\begin{table*}[!t]
\centering
\caption{Per-token \textbf{VRAM} storage cost per KV head ($d{=}128$, block size $B{=}16$, value group size $g{=}16$).  Tier~2 (CPU pinned RAM) stores FP16 originals at 512 bytes/token and is not included in this table.}
\label{tab:storage}
\small
\begin{tabular}{lrrr}
\toprule
\textbf{Component} & \textbf{Bytes/token} & \textbf{vs.\ FP16} & \textbf{Notes} \\
\midrule
\multicolumn{4}{l}{\emph{Dense FP16 (VRAM)}} \\
Keys (FP16)   & 256 & --- & $d \times 2$ \\
Values (FP16) & 256 & --- & $d \times 2$ \\
\textbf{Total} & \textbf{512} & \textbf{100\%} & \\
\midrule
\multicolumn{4}{l}{\emph{Tier 1 --- Certified (VRAM)}} \\
Keys (INT8)        & 128 & & $d \times 1$ \\
Key scales/offsets & 64  & & $2d \times 4 / B = 2 \times 128 \times 4 / 16$ \\
Values (INT4)      & 64  & & $d / 2$ \\
Value scales/offsets & 32  & & $2(d/g) \times 2 = 2 \times 8 \times 2$ (FP16) \\
Value $\eta_b$     & $<1$ & & 1 float per block, amortised over $B$ tokens \\
\textbf{Tier 1 total} & \textbf{288} & \textbf{56\%} & \\
\bottomrule
\end{tabular}
\end{table*}

The Tier~1 VRAM cost is approximately 56\% of dense FP16, a 44\% VRAM reduction.  The metadata overhead (scales and offsets) is significant---64 bytes per token for key metadata alone---but is amortised across the block and is the price of per-channel quantization quality.  The total memory footprint (VRAM + system RAM) is 288 + 512 = 800 bytes/token/head, or 156\% of dense---the system trades system RAM for VRAM, which is the scarcer resource.

The Tier-1 savings are context-independent at 44\%.  For example, at 128K context (32 layers, 8 KV heads, $d{=}128$): dense FP16 KV requires ${\sim}16$\,GB; Tier-1 requires ${\sim}9$\,GB.  However, the Tier-2 system RAM cost (FP16 originals for fallback) equals the dense size, and the runtime FP16 key scratch buffer further reduces the effective savings (see below).

\textbf{Runtime FP16 scratch buffer.}  At runtime, the system pages FP16 keys from Tier-2 into a VRAM scratch buffer to avoid repeated H2D transfers.  In theory, the per-layer GQA union (Section~\ref{sec:performance}, Table~\ref{tab:union}) approaches the position-corpus scale for full coverage.  In practice, the latency knee is far below that scale: at 64K (4096 position blocks per layer), a 1536-block cache reaches the latency plateau at 93.4\% hit rate (Table~\ref{tab:cache_sweep_64k}).  The asymmetric cache experiment (Table~\ref{tab:asymmetric_cache}) shows that key scratch dominates latency at 64K; reducing value scratch triggers a split-$K$ fallback that is cheap at 64K but expensive at 128K (Section~\ref{sec:performance}).  At 64K, the best operating point (2048 key / 256 value blocks, ${\sim}2.3$\,GB scratch) achieves 260\,ms---4.6\% faster than symmetric 1536/1536.  At 128K, the value-cache knee is 512 blocks: 2048 key / 512 value blocks (${\sim}2.6$\,GB scratch) matches symmetric 2048/2048 latency while saving 37\% of scratch VRAM.

Total runtime VRAM is Tier-1 (56\% of dense) + scratch.  Because the scratch is fixed-size, the VRAM ratio improves as context grows:

\begin{itemize}[leftmargin=*,itemsep=2pt]
\item \textbf{Tier-1 compressed KV} (VRAM): INT8 keys + INT4 values + metadata.  56\% of dense, context-independent.
\item \textbf{FP16 scratch buffer} (VRAM): key cache plus value cache (Table~\ref{tab:asymmetric_cache}).  Fixed size regardless of context; ${\sim}2.3$\,GB at the asymmetric 2048K/256V operating point (64K), ${\sim}2.6$\,GB at 2048K/512V (128K knee), or ${\sim}4.3$\,GB at symmetric 2048/2048.
\item \textbf{Tier-2 originals} (system RAM): full FP16 KV retained for fallback.  Equals dense cost.  Required for the certification guarantee.
\end{itemize}

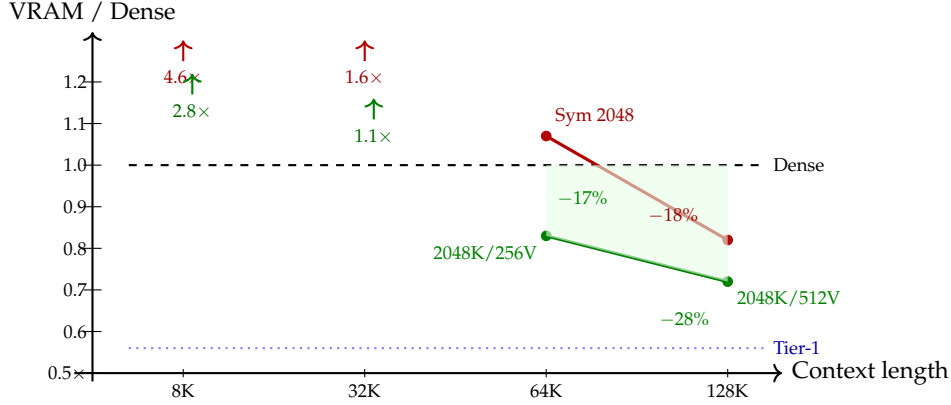
\begin{figure*}[!t]
\centering
\begin{tikzpicture}[x=2.4cm, y=5.5cm]
  \draw[thick,->] (-0.1,0.5) -- (3.8,0.5) node[right,font=\small] {Context length};
  \draw[thick,->] (0,0.48) -- (0,1.32) node[above,font=\small] {VRAM / Dense};

  \foreach \y/\lab in {0.5/0.5$\times$, 0.6/0.6, 0.7/0.7, 0.8/0.8, 0.9/0.9, 1.0/1.0, 1.1/1.1, 1.2/1.2} {
    \draw (-0.05,\y) -- (0.05,\y) node[left,font=\scriptsize,xshift=-2pt] {\lab};
  }
  \foreach \x/\lab in {0.5/8K, 1.5/32K, 2.5/64K, 3.5/128K} {
    \draw (\x,0.49) -- (\x,0.51) node[below,font=\scriptsize,yshift=-2pt] {\lab};
  }

  \draw[thick, black, dashed] (0.2,1.0) -- (3.7,1.0);
  \node[right,font=\scriptsize] at (3.7,1.0) {Dense};

  \draw[thick, blue!50, dotted] (0.2,0.56) -- (3.7,0.56);
  \node[right,font=\scriptsize,blue!60!black] at (3.7,0.56) {Tier-1};

  \draw[very thick, red!70!black]
    (2.5,1.07) -- (3.5,0.82);
  \filldraw[red!70!black] (2.5,1.07) circle (1.8pt);
  \filldraw[red!70!black] (3.5,0.82) circle (1.8pt);
  \draw[thick, red!70!black, ->] (0.5,1.25) -- (0.5,1.30);
  \node[font=\scriptsize,red!60!black,below] at (0.5,1.25) {4.6$\times$};
  \draw[thick, red!70!black, ->] (1.5,1.25) -- (1.5,1.30);
  \node[font=\scriptsize,red!60!black,below] at (1.5,1.25) {1.6$\times$};
  \node[font=\scriptsize,red!60!black,above right] at (2.5,1.07) {Sym 2048};

  \draw[very thick, green!50!black]
    (2.5,0.83) -- (3.5,0.72);
  \filldraw[green!50!black] (2.5,0.83) circle (1.8pt);
  \filldraw[green!50!black] (3.5,0.72) circle (1.8pt);
  \draw[thick, green!50!black, ->] (0.55,1.17) -- (0.55,1.22);
  \node[font=\scriptsize,green!40!black,below] at (0.55,1.17) {2.8$\times$};
  \draw[thick, green!50!black, ->] (1.55,1.11) -- (1.55,1.16);
  \node[font=\scriptsize,green!40!black,below] at (1.55,1.11) {1.1$\times$};
  \node[font=\scriptsize,green!50!black,below left] at (2.5,0.83) {2048K/256V};
  \node[font=\scriptsize,green!50!black,below right] at (3.5,0.72) {2048K/512V};

  \fill[green!10, opacity=0.6]
    (2.5,0.83) -- (3.5,0.72)
    -- (3.5,1.0) -- (2.5,1.0) -- cycle;

  \node[font=\scriptsize,green!50!black] at (2.7,0.92) {$-17\%$};
  \node[font=\scriptsize,green!50!black,left] at (3.45,0.63) {$-28\%$};
  \node[font=\scriptsize,red!60!black] at (3.2,0.88) {$-18\%$};
\end{tikzpicture}
\caption{Runtime VRAM vs.\ dense FP16, normalised (measured Tier-1 + scratch).  Below 64K the fixed scratch allocation exceeds dense KV and the system uses more VRAM than dense (arrows indicate off-chart ratios).  The symmetric 2048-block operating point (red) breaks even at 64K ($1.07\times$) and saves 18\% at 128K ($0.82\times$).  The asymmetric operating points (green; Table~\ref{tab:asymmetric_cache}) use the headline configurations at each context length: 2048K/256V at 64K ($0.83\times$, 17\% saving, negligible latency impact) and 2048K/512V at 128K ($0.72\times$, 28\% saving, matching symmetric latency at the value-cache knee).  The Tier-1 compressed storage (dotted blue, 56\% of dense) represents the theoretical floor if scratch were eliminated entirely.}
\label{fig:vram_scaling}
\end{figure*}

\subsection{Quality Delta Summary}

Table~\ref{tab:delta_summary} summarises the quality delta across all benchmark configurations.  PG-19 shows certified matching dense within noise at all context lengths (8K, 32K, 64K, 128K), with 95\% CIs including zero in every case.  RULER shows certified slightly exceeding dense at 64K ($+0.015$, 20 slices) and matching at 128K ($+0.004$, 20 slices, CI $[-0.028, +0.037]$); at 32K (20 slices), a mild $-0.019$ gap appears (95\% CI $[-0.044, +0.001]$), driven by CWE ($-0.100$) in the same pattern as the 8K degradation but smaller in magnitude.  The 8K gap ($-0.073$) is consistent with INT4 value reconstruction error as the dominant mechanism on VT/CWE/FWE (Section~\ref{sec:ruler_8k_ablation}).  NIAH at 8K, 32K, and 64K (100 paired trials each) shows no statistically significant degradation under exact McNemar tests: $p{=}1.0$ at 8K, $p{=}0.727$ at 32K, and $p{=}1.0$ at 64K.  At 128K the base model cannot retrieve (dense accuracy 0\%), so both systems agree trivially.

\begin{table}[ht]
\centering
\caption{Quality delta summary (Certified $-$ Dense) across all configurations.  Positive means certified outperforms dense.  8K: 20 PG-19 chunks, 100 NIAH trials, 20 RULER slices; 32K: 20 PG-19 chunks, 100 NIAH trials, 20 RULER slices; 64K: 20 PG-19 chunks, 100 NIAH trials, 20 RULER slices; 128K: 20 PG-19 chunks, 20 RULER slices (NIAH omitted---base model scores 0\% at 128K).}
\label{tab:delta_summary}
\footnotesize
\begin{tabular}{lrrrr}
\toprule
\textbf{Benchmark} & \textbf{8K} & \textbf{32K} & \textbf{64K} & \textbf{128K} \\
\midrule
PG-19 ($\Delta$ppl) & $-0.002$ & $+0.001$ & $+0.001$ & $+0.001$ \\
NIAH ($\Delta$acc)   & $0.0\%$ & $-2.0\%$ & $0.0\%$ & ---\textsuperscript{$\dagger$} \\
RULER ($\Delta$acc)  & $-0.069$ & $-0.019$ & $+0.015$ & $+0.004$ \\
\bottomrule
\end{tabular}

{\footnotesize \textsuperscript{$\dagger$}128K NIAH omitted: the base model scores 0\% on dense at this context length.}
\end{table}

\textbf{System telemetry.}  The score-consistency canary (Section~\ref{sec:instability}) was validated on a representative subset of runs (a dedicated canary pass at each context length) with zero violations, confirming that the Theorem~\ref{thm:key} $\Delta$ bound was never empirically exceeded.  The canary is disabled during production quality benchmarks for throughput; the retained run artifacts therefore have \texttt{score\_consistency\_check: false}.  We report two different trigger granularities.  A \emph{decode step} is one generated token; a \emph{head-step} is one query head in one layer for one generated token.  Rung~1 expansion and Rung~2 value promotion are evaluated as step-level policies and are active on nearly every 64K step because the conservative $e^{3\Delta}$ mass inflation and the local value heuristic usually request extra precision.  Rung~3 is different: it is a per-head fallback triggered only when the ranking certificate fails.  Table~\ref{tab:fallback_counts} normalises the trigger counts at 64K.

\begin{table}[ht]
\centering
\caption{Fallback trigger counts at 64K.  A head-step is one query head $\times$ one layer $\times$ one decode step; for LLaMA~3.1-8B (8 KV heads, 32 layers, 500 decode steps) this is 128{,}000.  Boundary-check corr. are a subset of head-steps where the FP16/INT8 score order at the promote/tail boundary is detected as inconsistent and repaired at runtime.}
\label{tab:fallback_counts}
\scriptsize
\begin{tabular}{llrrr}
\toprule
\textbf{Mechanism} & \textbf{Bench.} & \textbf{Count} & \textbf{Denom.} & \textbf{Rate} \\
\midrule
Bnd-check corr. & PG-19 & 10{,}221 & 128{,}000 & 8.0\% \\
Bnd-check corr. & NIAH  & 905     & 128{,}000 & 0.7\% \\
Rung~3 (head SDPA) & PG-19 & 1{,}536 & 128{,}000 & 1.2\% \\
Rung~3 (head SDPA) & NIAH  & 2{,}816 & 128{,}000 & 2.2\% \\
Rung~4 (all SDPA)  & PG-19 & 0       & 500       & 0.0\% \\
Rung~4 (all SDPA)  & NIAH  & 0       & 500       & 0.0\% \\
\bottomrule
\end{tabular}
\end{table}  The effective promoted set after Rung~1 is usually $2K_{\max}{=}256$ blocks ($K^*$ mean 218--223 at 64K), so the INT8-tail fraction stabilises at 87--88\%.

\subsection{Runtime Performance}
\label{sec:performance}

The quality results above show that the certified system matches dense on PG-19, NIAH, and long-context RULER, while the default INT4-value operating point has value-sensitive RULER gaps at 8K and 32K.  This subsection quantifies the runtime cost of certification and reports throughput, memory, and cache behaviour across context lengths from 8K to 128K.  All performance measurements use the native Blackwell attention backend with the same configuration as the quality runs (Section~\ref{sec:setup}), measured on a single RTX PRO 6000 WS with an AMD EPYC 9534 64-core CPU and 221\,GB system RAM over PCIe~5.0 $\times$16.

\textbf{Phase breakdown at 64K.}  Interpretation: the returned output is usually cheap to compute once precision decisions are fixed; the dominant overhead comes from certification orchestration, especially ranking-consistency checking and data movement.  Table~\ref{tab:phase_breakdown} profiles the certified decode step at 64K with a 2048-block FP16 scratch cache (the capacity knee from the cache sweep, ${\sim}99.6\%$ hit rate).

\begin{table}[ht]
\centering
\caption{Per-step phase breakdown (500 certified decode steps, 64K PG-19, native Blackwell kernel, 2048-block scratch cache).  The ranking-consistency check and H2D page-in together account for 39\% of step time; non-attention computation accounts for another 39\%.}
\label{tab:phase_breakdown}
\footnotesize
\begin{tabular}{lr}
\toprule
\textbf{Phase} & \textbf{\%} \\
\midrule
Non-attention (MLP, norms, proj.) & 39.0 \\
Ranking-consistency check & 28.3 \\
H2D page-in & 11.1 \\
Ph.~2: FP16 attend + INT4 dequant & 8.4 \\
Adaptive selection (LSE, top-$K$) & 5.2 \\
Value decompression (INT4$\to$FP16) & 3.9 \\
Ph.~1: INT8 scoring (block masses) & 2.4 \\
Rung~3 per-head dense fallback & 1.7 \\
\midrule
\textbf{Total} & \textbf{100.0} \\
\bottomrule
\end{tabular}
\end{table}

At realistic cache capacity, the ranking-consistency check dominates at 28.3\% of step time, followed by H2D page-in at 11.1\%.  Non-attention computation (MLP, norms, projections) accounts for 39\%---similar absolute time to the full-mirror configuration but a smaller fraction because the total step is longer.  The FP16 attend kernel contributes only 8.4\%, confirming that the overhead is dominated by certification orchestration and memory traffic, not the attention computation itself.  Phase-1 INT8 scoring remains cheap at 2.4\%.  The ranking-consistency check is therefore the dominant systems bottleneck and the primary target for future kernel fusion.

\textbf{FP16 scratch cache capacity sweep.}  Tables~\ref{tab:cache_sweep_64k} and~\ref{tab:cache_sweep_128k} show how certified latency varies with FP16 scratch cache capacity at 64K and 128K.

\begin{table}[ht]
\centering
\caption{FP16 scratch cache capacity sweep at 64K (PG-19, position corpus${}=4096$ blocks per layer).  Scratch capacity counts cache slots, not a full all-layer mirror; therefore a capacity equal to the position corpus is still not equivalent to retaining every FP16 original in VRAM\@.  The latency knee is at 1536 blocks (93.4\% hit rate); beyond this point, additional capacity yields diminishing returns.  ``Full mirror'' retains all FP16 originals in VRAM\@.}
\label{tab:cache_sweep_64k}
\scriptsize
\begin{tabular}{rlrrr}
\toprule
\textbf{Cap.} & \textbf{Note} & \textbf{ms} & \textbf{Hit\%} & \textbf{H2D MB} \\
\midrule
0       & no cache       & 3402 & 0.0\%   & 1480 \\
256     &                & 1485 & 0.0\%   & 1480 \\
512     &                & 1262 & 0.0\%   & 1480 \\
1024    &                & 824  & 21.3\%  & 1209 \\
1280    &                & 361  & 78.1\%  & 490 \\
\rowcolor[gray]{0.92}
1536    & latency knee   & 262  & 93.4\%  & 295 \\
1792    &                & 260  & 96.7\%  & 245 \\
2048    & $0.5\times$ corpus & 259 & 97.5\% & 239 \\
4096    & position corpus & 270 & 97.5\% & 240 \\
\rowcolor[gray]{0.92}
$\infty$ & full mirror  & 174  & 100\%   & 0 \\
\bottomrule
\end{tabular}
\end{table}

\begin{table}[ht]
\centering
\caption{FP16 scratch cache capacity sweep at 128K (PG-19, position corpus${}=8192$ blocks per layer).  Scratch capacity counts cache slots, not a full all-layer mirror; the same knee appears at 2048 blocks (cache hit rate 89.8\%), with diminishing returns beyond.}
\label{tab:cache_sweep_128k}
\scriptsize
\begin{tabular}{rlrrr}
\toprule
\textbf{Cap.} & \textbf{Note} & \textbf{ms} & \textbf{Hit\%} & \textbf{H2D GB} \\
\midrule
0       & no cache       & 4187 & 0.0\%   & 2.05 \\
256     &                & 2669 & 0.0\%   & 2.04 \\
1024    &                & 1588 & 11.6\%  & 1.87 \\
\rowcolor[gray]{0.92}
2048    &                & 319  & 89.8\%  & 0.76 \\
4096    & $0.5\times$ corpus & 322 & 96.3\% & 0.66 \\
5120    &                & 330  & 96.3\%  & 0.65 \\
8192    & position corpus & 334 & 96.3\% & 0.66 \\
\rowcolor[gray]{0.92}
$\infty$ & full mirror  & 219  & 100\%   & 0.0 \\
\bottomrule
\end{tabular}
\end{table}

A focused knee sweep at 64K (Table~\ref{tab:cache_sweep_64k}) reveals that the latency plateau begins at 1536 blocks (93.4\% hit rate, 262\,ms), not 2048: increasing capacity from 1536 to 2048 reduces latency by less than 1\%.  Below 1536, latency rises steeply---1280 blocks gives 361\,ms (78\% hit), and 1024 gives 824\,ms.  The transition is sharp: the system jumps from $3\times$ the plateau latency to the plateau within a $1.5\times$ capacity increase.  At 128K (Table~\ref{tab:cache_sweep_128k}), the same pattern appears at 2048 blocks (89.8\% hit rate).  The remaining gap between the plateau and full mirror (262 vs.\ 174\,ms at 64K; 319 vs.\ 219\,ms at 128K) reflects the cost of ongoing H2D page-in for blocks evicted and re-promoted each step.

A 1536-block cache has a direct VRAM benefit: scratch drops from 4.30\,GB to ${\sim}3.22$\,GB, bringing the 64K total to ${\sim}0.94\times$ dense---6\% VRAM savings instead of 7\% overspend.  This makes the VRAM crossover point 64K rather than 128K for deployments that can tolerate the ${\sim}4\times$ latency overhead.

\textbf{Asymmetric key/value scratch cache.}  The ranking-consistency check (28.3\% of step time) and Phase~2 attend both require FP16 keys; when the value error bound triggers FP16 value promotion (Rung~2), the promoted values must also fit in scratch.  Table~\ref{tab:asymmetric_cache} independently varies key and value scratch cache sizes.  When the value cache is too small to hold all promoted FP16 values for a step, the kernel falls back to a split-$K$ path that processes each promoted block individually; the ``Fallback'' column reports the fraction of steps using this slower path.

\begin{table}[ht]
\centering
\caption{Asymmetric FP16 scratch cache (PG-19, 500 decode steps, corrected value-promotion path).  At 64K the split-$K$ fallback is cheap ($<$5\% latency impact).  At 128K the value-cache knee is 512 blocks: above this capacity the split-$K$ fallback adds no measurable latency; below it (256 blocks) latency nearly doubles.}
\label{tab:asymmetric_cache}
\footnotesize
\begin{tabular}{lrrrr}
\toprule
\textbf{Ctx} & \textbf{Keys} & \textbf{Values} & \textbf{Cert ms} & \textbf{Fallback} \\
 & (blk) & (blk) & & \\
\midrule
64K & 1536 & 2048 & 267.8 & 0\% \\
64K & 1536 & 1536 & 272.6 & 0\% \\
64K & 1536 & 1024 & 274.0 & 8\% \\
64K & 1536 & 768  & 272.8 & 100\% \\
64K & 1536 & 512  & 272.3 & 100\% \\
64K & 1536 & 256  & 269.8 & 100\% \\
64K & 2048 & 512  & 262.5 & 100\% \\
\rowcolor[gray]{0.92}
64K & 2048 & 256  & 260.0 & 100\% \\
\midrule
128K & 2048 & 2048 & 342.5 & 0\% \\
128K & 2048 & 1792 & 340.2 & 1\% \\
128K & 2048 & 1536 & 341.2 & 58\% \\
128K & 2048 & 1280 & 341.2 & 100\% \\
128K & 2048 & 1024 & 339.1 & 100\% \\
\rowcolor[gray]{0.92}
128K & 2048 & 512  & 342.6 & 100\% \\
128K & 2048 & 256  & 587.8 & 100\% \\
\bottomrule
\end{tabular}
\end{table}

At 64K, the split-$K$ fallback triggers at $\leq$768 value blocks but adds negligible latency: all configurations from 256 to 2048 value blocks span only 260--274\,ms, a $<$5\% range.  Key cache size remains the dominant factor: going from 1536 to 2048 key blocks reduces latency by ${\sim}10$\,ms regardless of value cache capacity.  The best 64K configuration (2048\,keys / 256\,values, 260.0\,ms) is 4.6\% faster than symmetric 1536/1536 (272.6\,ms).

At 128K, the value-cache knee is at 512 blocks.  Configurations with $\geq$512 value blocks (512, 1024, 1280, 1536, 2048) all achieve ${\sim}340$\,ms regardless of whether the split-$K$ fallback fires---the fallback is cheap when the value hit rate is sufficient ($\geq$72\%).  Below the knee, at 256 value blocks, latency jumps to 587.8\,ms ($1.72\times$ symmetric), indicating that the per-step value working set exceeds the cache capacity and forces excessive H2D traffic through the fallback path.  The best 128K asymmetric configuration is 2048\,keys / 512\,values (${\sim}2.6$\,GB scratch), matching symmetric 2048/2048 latency (342.6 vs.\ 342.5\,ms) while saving 37\% of scratch VRAM (2560\,MB vs.\ 4096\,MB).

\textbf{Attention backend.}  The certified system uses a native Blackwell attention backend that partitions the sequence across thread blocks (analogous to FlashDecoding~\cite{flashinfer}), maintains FP32 online-softmax state within each partition, and reduces across partitions in a second pass.  The per-call profile at 64K shows 0.50\,ms for the partial-sequence passes and 0.01\,ms for the cross-partition reduce, totalling 0.51\,ms per kernel call across 16{,}000 calls (32 layers $\times$ 500 steps).  The backend is independent of the error bounds and fallback logic.

\textbf{Context-scaling latency.}  Table~\ref{tab:context_scaling} reports per-step latency across context lengths using a fixed 2048-block FP16 scratch cache---the realistic operating point where the system saves VRAM at long contexts.

\begin{table*}[!t]
\centering
\caption{Decode latency vs.\ context length (LLaMA~3.1-8B, RTX PRO 6000 WS, 2048-block scratch cache, 500 timed decode steps).  ``Cache hit'' is the fraction of promoted blocks served from VRAM scratch vs.\ paged in from system RAM\@.  H2D is per-step host-to-device transfer volume.}
\label{tab:context_scaling}
\small
\begin{tabular}{rrrrrrr}
\toprule
\textbf{Context} & \textbf{Dense ms} & \textbf{Cert ms} & \textbf{Ratio} & \textbf{$K^*$ mean} & \textbf{Cache hit} & \textbf{H2D MB} \\
\midrule
8K   & 60.9 & 166.3  & $2.73\times$ & 159 & 100\%  & 0 \\
32K  & 61.2 & 232.2  & $3.79\times$ & 217 & 99.7\% & 114 \\
64K  & 62.6 & 257.4  & $4.11\times$ & 222 & 99.6\% & 251 \\
128K & 72.8 & 346.7  & $4.76\times$ & 226 & 87.7\% & 644 \\
\bottomrule
\end{tabular}
\end{table*}

The certified-to-dense ratio increases monotonically with context length, driven by H2D page-in traffic: at 8K the 2048-block cache covers the entire 512-block corpus (100\% hit, zero H2D), while at 128K only 88\% of requests hit the scratch cache and each step pages in ${\sim}644$\,MB from system RAM over PCIe~5.0.  Dense step time is nearly flat from 8K to 64K ($60.9 \to 62.6$\,ms) but rises at 128K ($72.8$\,ms) as the KV cache exceeds L2 capacity.  With full-mirror caches (all FP16 originals in VRAM), the 64K ratio drops to ${\sim}2.4\times$ (Table~\ref{tab:cache_sweep_64k}), confirming that H2D is the dominant scaling cost---but full-mirror uses more VRAM than dense (Table~\ref{tab:memory}).

\textbf{Remaining overhead.}  The phase breakdown (Table~\ref{tab:phase_breakdown}) shows that at the 2048-block operating point, non-attention computation accounts for 39\% of the certified step, while certification-related overhead consumes 61\%: ranking-consistency check (28.3\%), H2D page-in (11.1\%), FP16 attend (8.4\%), adaptive selection (5.2\%), value decompression (3.9\%), INT8 scoring (2.4\%), and Rung~3 boundary recompute (1.7\%).  The ranking check alone is the single largest target for optimisation; fusing it into the Phase-2 attend kernel or batching across heads would address 28\% of step time.  The asymmetric cache experiment (Table~\ref{tab:asymmetric_cache}) shows that at 64K, maximising key cache capacity is more effective than symmetric allocation; reducing value cache triggers a split-$K$ fallback that is cheap at 64K but expensive at 128K.

\textbf{GQA union saturation.}  The cache sweeps are shaped by GQA head diversity and by the fact that a finite scratch cache is not a full all-layer FP16 mirror.  Each of $H_Q{=}32$ query heads independently selects its own top-$K^*$ blocks per layer.  The FP16 key cache is shared across all KV groups within each layer, so the per-layer working set is the union across all 32 query heads (not merely the $G{=}4$ heads within a single KV group).  With Rung~1 expansion, the effective $K^*$ per query head is $\min(2K_{\max}, \lceil N/B \rceil)$.  The per-layer FP16 working set is:
\begin{equation}
\label{eq:union}
  U = \!\left\lceil \frac{N}{B} \right\rceil\!
  \left(1 - \!\left(1 - \frac{\min(2K_{\max}, \lceil N/B \rceil)}
  {\lceil N/B \rceil}\right)^{\!H_Q}\right)
\end{equation}

Table~\ref{tab:union} evaluates this for LLaMA~3.1-8B
with $K_{\max}{=}128$, $H_Q{=}32$, $B{=}16$,
and Rung~1 active
(effective per-head selection $K^*_{\mathrm{eff}} = \min(256,\allowbreak \lceil N/B \rceil)$).
VRAM saving is Tier-1 (56\% of dense) plus the
union-sized FP16 key cache ($U \times 50\%$ of dense),
compared against dense FP16 KV\@.

\begin{table*}[!t]
\centering
\caption{GQA union analysis ($H_Q{=}32$ query heads).  The 64K union covers ${\sim}87\%$ of position blocks per layer under the expanded $K^*$ policy, explaining why a full all-layer mirror is much larger than the finite scratch caches used in the latency sweeps.}
\label{tab:union}
\small
\begin{tabular}{rrrrrr}
\toprule
\textbf{Context} & \textbf{Blocks} & \textbf{$K^*_{\mathrm{eff}}/N_B$} & \textbf{Union} & \textbf{Dense} & \textbf{Saving} \\
 & $\lceil N/B \rceil$ & & & (MB) & \\
\midrule
8K   & 512   & 50\%    & ${\approx}100\%$  & 1024  & $-6\%$ \\
32K  & 2048  & 12.5\%  & 99\%    & 4096  & $-5\%$ \\
\rowcolor[gray]{0.92}
64K  & 4096  & 6.2\%   & 87\%    & 8192  & $+1\%$ \\
128K & 8192  & 3.1\%   & 64\%    & 16384 & $+12\%$ \\
256K & 16384 & 1.6\%   & 39\%    & 32768 & $+25\%$ \\
\bottomrule
\end{tabular}
\end{table*}

The union analysis reveals that with a position-corpus-capacity scratch buffer, VRAM savings are negligible at 64K and only become substantial ($>$10\%) at 128K+ context.  However, the cache sweep (Table~\ref{tab:cache_sweep_64k}) shows that scratch capacity need not match the corpus: a symmetric 1536-block cache achieves the latency plateau at 64K (${\sim}0.94\times$ dense).  With an asymmetric cache (Table~\ref{tab:asymmetric_cache}), the value cache can be reduced to 256 blocks at 64K or 512 blocks at 128K at zero latency cost, bringing 64K to ${\sim}0.83\times$ dense and 128K to ${\sim}0.72\times$ dense.  Below 64K, the union covers virtually the entire corpus and the architecture uses \emph{more} VRAM than dense regardless of cache sizing.

\textbf{Memory profile.}  Table~\ref{tab:memory} reports measured VRAM and system RAM usage at the 2048-block scratch cache operating point.  The Tier-1 ratio (INT8 keys + INT4 values + metadata) is consistently ${\sim}0.565\times$ dense, matching the theoretical 56\% from the storage analysis (Section~\ref{sec:storage}).  The scratch cache is fixed at 4.30\,GB regardless of context length, so total VRAM decreases relative to dense as context grows.

\begin{table*}[!t]
\centering
\caption{Measured memory usage (LLaMA~3.1-8B, 2048-block scratch cache).  Tier-1 stores quantized KV; scratch is fixed-size FP16 cache for promoted blocks; Tier-2 (system RAM) stores FP16 originals for fallback.  Total VRAM = Tier-1 + scratch.  At 128K, the certified system uses 18\% less VRAM than dense.}
\label{tab:memory}
\small
\begin{tabular}{rrrrrr}
\toprule
\textbf{Context} & \textbf{Dense VRAM} & \textbf{Tier-1 VRAM} & \textbf{Scratch VRAM} & \textbf{Total VRAM} & \textbf{System RAM} \\
 & (GB) & (GB) & (GB) & (vs.\ dense) & (GB) \\
\midrule
8K   & 1.07 & 0.61 & 4.30 & $4.57\times$ & 1.08 \\
32K  & 4.29 & 2.43 & 4.30 & $1.57\times$ & 4.30 \\
64K  & 8.59 & 4.85 & 4.30 & $1.07\times$ & 8.59 \\
\rowcolor[gray]{0.92}
128K & 17.18 & 9.71 & 4.30 & $0.82\times$ & 17.18 \\
\bottomrule
\end{tabular}
\end{table*}

Because scratch is constant, the crossover where certified uses \emph{less} VRAM than dense occurs between 64K and 128K.  At 128K with symmetric caches, total certified VRAM (14.0\,GB) is 82\% of dense (17.2\,GB), saving 3.2\,GB---at the cost of 17.2\,GB system RAM for Tier-2 and $4.76\times$ latency overhead.  With an asymmetric cache (2048K/512V), scratch shrinks from 4.30\,GB to ${\sim}2.6$\,GB, bringing the 128K total to ${\sim}0.72\times$ dense---a 28\% saving at no latency cost (Table~\ref{tab:asymmetric_cache}).  At 64K the system is approximately break-even on VRAM ($1.07\times$) with symmetric 2048-block caches; the asymmetric 2048K/256V configuration shrinks scratch to ${\sim}2.3$\,GB, bringing the 64K total to ${\sim}0.83\times$ dense---a 17\% saving with negligible latency impact.  Below 64K, the fixed scratch allocation exceeds the dense KV cache regardless of cache policy.  System RAM for Tier-2 (FP16 originals) closely matches dense VRAM at all context lengths.

\textbf{Per-KV-group selection.}  Collapsing the 32 independent per-query-head selections into 8 per-KV-group selections (union within each group of $G{=}4$ query heads) reduces the GQA fan-out.  At 64K with sub-corpus cache (cap${}=1024$), per-KV-group selection reduced H2D volume by 11\% and improved throughput by 24\%, confirming the fan-out hypothesis.  This optimisation is particularly relevant at the 2048-block operating point where H2D page-in accounts for 11\% of step time (Table~\ref{tab:phase_breakdown}).

\subsection{Numerical Precision of the Attention Kernel}
\label{sec:numerical}

The certified attention kernel uses a different computational path from Flash Attention (tensor-core TF32/FP16 matmul with FP32 accumulation), which PyTorch's \torchsdpa{} dispatches to.  These two paths are mathematically equivalent but not numerically identical: different rounding in the inner product produces ${\sim}$1e-7 max absolute difference per block per head.

For language modelling, this difference is negligible---it vanishes into the per-token log-probability noise.  For retrieval-adversarial benchmarks (NIAH), where a single token's attention weight determines whether the model ``finds'' the needle, the compounded rounding difference across 32 layers can flip marginal retrievals.

All quality results in this paper use FP32 online-softmax accumulators.  A development experiment with FP64 scalar accumulators ($m$, $\ell$ in FP64, output accumulator $o$ in FP32) recovered $\sim$2pp on NIAH 8K (Table~\ref{tab:niah_ablation}), but the FP32 path was sufficient for all reported quality metrics and is what the retained run artifacts use.  FP64 scalar accumulation remains available as an option for deployments where the arithmetic-path gap is a concern.

The remaining $\sim$4pp gap is attributable to differences in the matmul path (certified kernel vs.\ Flash Attention tensor cores), not the quantization scheme.  A tensor-core-native certified kernel that matched Flash Attention's rounding may reduce this gap, though the magnitude of any residual is unknown without implementation.  This does not affect the certification guarantee: Rung~3/4 fallbacks bypass the certified kernel entirely and call \torchsdpa{} (Section~\ref{sec:fallback}).

\section{Discussion}
\label{sec:discussion}

The primary contribution of this work is not compression itself---KV cache quantization is well-established---but the \emph{certification framing}: formal per-head, per-step error bounds coupled with runtime monitoring and unconditional FP16 fallback.  The tiered cache provides the compression savings of INT8 keys and INT4 values (44\% Tier-1 VRAM reduction; runtime savings depend on the FP16 key scratch buffer needed for the GQA union working set---see Table~\ref{tab:union}) while retaining the error profile of full-precision attention as a guaranteed fallback.

\textbf{The cost of the guarantee.}  With a symmetric 2048-block scratch cache, the certified system runs at $2.7\text{--}4.8\times$ the latency of dense Flash Attention across 8K--128K context (Table~\ref{tab:context_scaling}), with the overhead increasing at longer contexts due to H2D page-in traffic.  At 64K, an asymmetric cache (2048 key / 256 value blocks) achieves 260\,ms ($4.2\times$ dense)---4.6\% faster than symmetric 1536/1536---because key scratch dominates latency; the split-$K$ value fallback adds negligible cost at this context length (Table~\ref{tab:asymmetric_cache}).  At 128K, the value-cache knee is 512 blocks: 2048K/512V (${\sim}2.6$\,GB scratch) matches symmetric latency while saving 37\% of scratch; below 512 blocks the split-$K$ fallback becomes expensive (Section~\ref{sec:performance}).  The kernel overhead is an implementation artefact amenable to further fusion and caching optimisations, independent of the error bounds or fallback logic.  The H2D component ($11\%$ of step time at 64K) is an inherent cost of the sub-corpus cache architecture and can be traded against VRAM by increasing scratch capacity (Tables~\ref{tab:cache_sweep_64k}--\ref{tab:cache_sweep_128k}).  At 128K with symmetric 2048 blocks, the certified system saves 18\% VRAM vs.\ dense ($0.82\times$, Table~\ref{tab:memory}); retaining FP16 originals in system RAM (Tier-2) costs $\sim$17\,GB.  The fallback ladder is active but not uniformly at the same granularity: Rung~1 expansion and Rung~2 value promotion are near-universal step-level policies at 64K, while Rung~3 is a sparse per-head fallback that adds only 1.7\% of step time at 64K.

\textbf{Current operating regime.}  Today, the system's value proposition is narrow but clear: long-context inference (64K+) where the dense KV cache strains VRAM and quality regressions are unacceptable.  At shorter contexts the fixed scratch buffer exceeds the dense cache size, and the $2.7\text{--}4.8\times$ latency overhead makes it unsuitable when dense FP16 fits comfortably.  The certification architecture is general, but the current operating point is shaped by implementation-specific costs---not fundamental limits.

\textbf{Axes of improvement.}  Several independent developments could shift the operating point substantially, without changing the certification architecture or the fallback guarantee:

\emph{Advanced quantization schemes.}  The current INT8 key / INT4 value scheme with per-channel metadata yields 56\% Tier-1 storage (Table~\ref{tab:storage}), but the certification framework is not tied to this specific format.  The error decomposition (Section~\ref{sec:bounds}) applies to any block-quantized representation that admits a deterministic reconstruction bound.  Integrating recent advances in KV cache quantization---the 2-bit asymmetric schemes of KIVI~\cite{kivi}, the sub-4-bit per-channel techniques of KVQuant~\cite{kvquant}, the online vector quantization of TurboQuant~\cite{turboquant}, or the hardware-aware group-wise layouts of InnerQ~\cite{innerq} (Table~\ref{tab:positioning})---into the certified framework is a natural research direction.  Each brings a different compression--error tradeoff; the open question is whether the resulting bounds remain tight enough to keep fallback rates low and the guarantee practically useful.

\emph{Tighter error bounds.}  The current $e^{3\Delta}$ mass inflation (Section~\ref{sec:scoring}) is conservative: it accounts for INT8 query scoring headroom that a tensor-core-native kernel could avoid.  Reducing this to $e^{2\Delta}$ or below would shrink the promoted set, reduce H2D traffic, and lower the latency overhead.  Similarly, a global value-error budget (rather than the current local promotion heuristic) would enable principled quality--latency tradeoffs.

\emph{Memory hierarchy research.}  The H2D page-in cost (11\% of step time at 64K, increasing at 128K) is bounded by PCIe bandwidth.  Emerging memory technologies---CXL-attached memory pools, coherent GPU-host fabrics, and multi-tier HBM architectures---could fundamentally change the cost structure of tiered caches.  Research into cache replacement policies, predictive prefetching of FP16 blocks based on attention-pattern forecasting, and hardware-software co-design for certified inference would shift the latency knee and broaden the viable operating regime.

\textbf{Two-term independence.}  The error decomposition's most practically useful property is that key compression error and value compression error are independent---they have separate bounds, separate monitors, and separate escalation paths.  The tolerance sweep (varying $v_{\mathrm{tol}}$ from 0.5 to 0.05) confirms this empirically: the adaptive top-$K$ selection ($K^*$) is invariant to value tolerance changes, while $\eta$-driven value escalation scales with $1/v_{\mathrm{tol}}$.  This decoupling simplifies configuration: the key-side and value-side precision budgets can be tuned independently.

\textbf{Non-monotonicity of precision vs.\ retrieval accuracy.}  The NIAH ablation (Section~\ref{app:niah_ablation}) reveals that increasing the FP16 key budget does not monotonically improve retrieval accuracy.  The mechanism is attention dilution: promoting tail blocks to FP16 corrects their scores upward, increasing their softmax share at the expense of the needle token's share.  This is a general property of non-uniform precision attention, not specific to our system.  The ranking-consistency check (Section~\ref{sec:ranking_consistency}) detects this condition at runtime and triggers per-head fallback to dense attention, converting a potential quality regression into a bounded compute cost.

\textbf{Statistical power and the NIAH result.}  At both 32K and 64K (100 paired trials each), exact McNemar tests show no statistically significant degradation: $p{=}0.727$ at 32K, $p{=}1.0$ at 64K.  The 64K result is particularly clean: identical 66\%/66\% accuracy with perfectly balanced discordance (2 critical, 2 reverse).  This underscores the importance of adequate statistical power when evaluating retrieval benchmarks, where per-trial variance is high and effect sizes are small---an earlier 20-trial run showed an apparent $-5\%$ gap that disappeared entirely at $n{=}100$.  The NIAH precision ablation (Appendix~\ref{app:niah_ablation}) decomposes the sources of variance; the final system's aggregate effect on retrieval accuracy is within noise at the operating point.

\textbf{Conservative mass estimation.}  The implementation's $e^{3\Delta} \approx 1.72\times$ substitution (Section~\ref{sec:scoring}) means the selector promotes more blocks to FP16 than strictly necessary under FP16 query scoring ($e^{2\Delta} \approx 1.43\times$), wasting page-in bandwidth but not introducing error.

\textbf{Architecture independence.}  The tiered cache, error decomposition, and fallback ladder are agnostic to the model architecture.  They require only standard scaled dot-product attention with a block-organised KV cache.  The per-channel INT8 key quantization and per-group INT4 value quantization adapt to different head dimensions and value distributions without model-specific tuning.

\section{Limitations}
\label{sec:limitations}

The certified guarantee is \emph{per-head, per-step}: it bounds $\norm{O_{\mathrm{quant}} - O_{\mathrm{ref}}}$ (Section~\ref{sec:decomposition}), not the distance to $O_{\mathrm{dense}}$.  It does not provide an end-to-end model-quality guarantee.  On NIAH at 64K (100 paired trials), dense and certified achieve identical 66\% accuracy with no statistically significant difference (exact McNemar $p{=}1.0$), but the NIAH precision ablation (Table~\ref{tab:niah_ablation}) demonstrates that individual precision components (INT4 values, INT8 tail keys, kernel numerical path) each contribute measurable variance at the per-trial level.  The aggregate effect cancels to noise at the operating point, but could become significant on tasks with sharper retrieval requirements or longer autoregressive chains.  The aggregate effect on model quality must therefore be assessed empirically per deployment.

The system RAM requirement for Tier~2 (FP16 originals) equals the dense KV cache size---e.g., $\sim$16\,GB at 128K context on LLaMA~3.1-8B.  This is the price of the unconditional fallback guarantee.  Deployments with limited system RAM must either accept the system RAM cost or forgo the fallback (losing the ``certified'' property and reverting to standard quantization).

At a realistic 2048-block scratch cache, the runtime overhead ranges from $2.7\times$ (8K) to $4.8\times$ (128K) dense latency (Table~\ref{tab:context_scaling}).  The dominant overhead components are the ranking-consistency check (28\% of step time at 64K) and H2D page-in (11\%), both amenable to optimisation (Section~\ref{sec:performance}).  VRAM savings materialise at 64K+ context when cache capacity is tuned: with symmetric caches, a 1536-block cache at 64K achieves ${\sim}0.94\times$ dense; with an asymmetric cache (2048 key / 256 value blocks, Table~\ref{tab:asymmetric_cache}), 64K drops to ${\sim}0.83\times$ dense with negligible latency impact.  At 128K, the asymmetric knee is 512 value blocks: 2048K/512V matches symmetric 2048/2048 latency while reducing scratch by 37\%; below 512 blocks the split-$K$ value fallback becomes expensive (Section~\ref{sec:performance}).  Below 64K the scratch allocation exceeds dense VRAM regardless of capacity.

The certified attention kernel uses a different computational path from Flash Attention, producing slightly different rounding (Section~\ref{sec:numerical}).  The development-time ablation (Table~\ref{tab:niah_ablation}) attributes $\sim$4pp to this kernel path; however, the 64K evaluation (100 paired trials, exact McNemar $p{=}1.0$) shows no statistically significant aggregate effect at the operating point.  A tensor-core-native kernel that matched Flash Attention's rounding may reduce this variance, though the magnitude of any residual is unknown without implementation.

The value group-size cliff between $g{=}16$ (near-lossless on perplexity) and $g{=}32$ (severe, $+26.4$ $\Delta$ppl) means there is no intermediate operating point for INT4 values at $d{=}128$.  The RULER ablations (Section~\ref{sec:ruler_8k_ablation}) demonstrate that even $g{=}16$ introduces measurable degradation on value-sensitive subtasks at short contexts, though this effect diminishes with context length (mild CWE residual at 32K, undetectable at 64K+).  At both 8K and 32K, FP16-value ablations close the gap, and tightening $v_{\mathrm{tol}}$ to 0.02 at 32K removes the gap in a targeted ablation (5 slices), confirming the mechanism and providing a deployment knob.  Deployments requiring finer-grained value compression tradeoffs would need INT8 values (not currently implemented in the kernel) or a different head dimension.

\section{Conclusion}
\label{sec:conclusion}

We presented a tiered KV cache architecture that stores INT8 keys and INT4 values in GPU memory (Tier-1: 56\% of FP16; runtime VRAM includes an additional FP16 key scratch buffer sized by the GQA working-set union) while retaining full-precision originals in system RAM (Tier-2) for runtime fallback.  A two-term error decomposition provides independent, per-head, per-step bounds on key compression error and value compression error, each with its own runtime monitor and escalation path.  An adaptive top-$K$ precision selector promotes high-mass blocks to FP16 keys based on the actual attention pattern at each decode step, and a four-rung fallback ladder terminates in the Dense baseline's own \torchsdpa{} code path with unmodified FP16 KV.  The certification is local (per-head, per-step) and does not imply end-to-end task-correctness guarantees.

The system introduces a ranking-consistency check that detects when INT8 scoring noise distorts the attention distribution and triggers per-head fallback to dense attention.  This mechanism is critical: it converts a previously unaddressed silent failure mode into a detectable and recoverable event.  On retrieval-adversarial benchmarks, non-uniform precision across the sequence can otherwise cause attention dilution that degrades accuracy even when per-step error bounds are satisfied.

On language modelling (PG-19), the certified system matches dense perplexity within measurement noise across all tested context lengths (8K--128K): $\Delta\mathrm{ppl}$ ranges from $-0.002$ to $+0.001$ with 95\% CIs including zero at every context length.  On structured long-context tasks (RULER, 7 subtasks), the certified system matches dense at 64K ($+0.015$, 20 slices) and 128K ($+0.004$, 20 slices), with the four RULER NIAH retrieval variants at or near dense performance across context lengths.  At 32K, a mild $-1.9$pp gap driven by CWE emerges (20 slices, 95\% CI nearly excludes zero); at 8K, a $-6.9$pp RULER gap is localised to three value-sensitive subtasks (VT, CWE, FWE) and is strongly consistent with INT4 value reconstruction error as the dominant mechanism---FP16-value ablations close the gap at both context lengths (8K to $-0.4$pp with 0 critical failures, 32K to $-0.5$pp), and tightening $v_{\mathrm{tol}}$ to 0.02 at 32K removes the gap in a targeted ablation (Section~\ref{sec:ruler_8k_ablation}).  On multi-needle retrieval (NIAH, 10 needles per trial, 100 paired trials each at 8K, 32K, 64K), no context length shows a statistically significant difference: exact McNemar $p{=}1.0$ at 8K, $p{=}0.727$ at 32K, and $p{=}1.0$ at 64K.  We observe and transparently report a residual kernel-path gap attributable to numerical precision differences between the certified kernel and Flash Attention, which is a property of the kernel implementation rather than the quantization scheme.  The ranking-consistency check recovers the majority of this gap; a tensor-core-native kernel that matched Flash Attention's rounding may reduce the remainder, though the magnitude of any residual is unknown without implementation.

A naive INT8K/INT4V baseline---identical quantization but with all certification and fallback disabled---substantially degrades retrieval (NIAH accuracy drops to 5--10\% vs.\ 39--66\% dense) and structured reasoning (RULER $-44$ to $-56$pp at all context lengths, VT 0\% everywhere).  The certified system recovers nearly all of this loss, indicating that the certification and fallback machinery is not merely monitoring overhead, but a primary mechanism by which the system preserves quality under aggressive quantization (Section~\ref{sec:naive}).

At a realistic 2048-block scratch cache, the certified system runs at $2.7\text{--}4.8\times$ the latency of dense Flash Attention across 8K--128K context (Table~\ref{tab:context_scaling}), with the overhead dominated by the ranking-consistency check and H2D page-in.  The kernel overhead is an implementation artefact independent of the error bounds; the remaining overhead is amenable to fusion and caching optimisations (Section~\ref{sec:performance}).  Meaningful VRAM savings materialise at 64K+ when scratch capacity is tuned: ${\sim}0.83\times$ dense at 64K with an asymmetric 2048K/256V cache, or ${\sim}0.72\times$ at 128K with 2048K/512V---both at no latency cost relative to symmetric caches (Table~\ref{tab:asymmetric_cache}).

The approach requires no model retraining or dataset-specific calibration and applies to any transformer with standard attention.  More broadly, it reframes KV cache quantization as a runtime-verified computation rather than a fixed approximation.

\section*{Artifact Availability}

Source code (CUDA/Triton kernels, certified attention implementation, and benchmark scripts), raw result JSONs, per-trial pass/fail matrices, and the \LaTeX{} source of this paper are available at:
\begin{center}
\url{https://github.com/DeanoC/certified-quantized-attention}
\end{center}
An archived snapshot is deposited on Zenodo~\cite{zenodo_artifact} with DOI \href{https://doi.org/10.5281/zenodo.19915933}{\texttt{10.5281/zenodo.19915933}}.
The repository includes the exact model configuration, quantization parameters, and seeds used for all reported experiments.  The version corresponding to this paper is tagged \texttt{arxiv-v1}; all results can be reproduced from that tag.  The repository is released under the MIT licence.


\appendix
\section{Total Variation Bound for Perturbed Softmax}
\label{app:tv_proof}

We derive the total variation bound used in the key compression error analysis.  The proof proceeds in three steps: bounding individual probability ratios, deriving the full-sequence TV bound, and specialising to the tail-only case used by the adaptive precision selector.

\textbf{Setup.}  Let $a_t = \exp(s_t)/Z$ and $a_t' = \exp(s_t')/Z'$ be two softmax distributions over $N$ tokens, where $|s_t - s_t'| \leq \Delta$ for all $t$, $Z = \sum_t \exp(s_t)$, and $Z' = \sum_t \exp(s_t')$.

\begin{lemma}[Probability ratio bound]
\label{lem:ratio}
For all $t$: $\; e^{-2\Delta} \leq a_t'/a_t \leq e^{2\Delta}$.
\end{lemma}

\begin{proof}
The ratio decomposes as:
\[
  \frac{a_t'}{a_t} = \frac{\exp(s_t')}{\exp(s_t)} \cdot \frac{Z}{Z'} = \exp(s_t' - s_t) \cdot \frac{Z}{Z'}
\]

\emph{Factor 1: per-token logit ratio.}  Since $|s_t - s_t'| \leq \Delta$, we have $e^{-\Delta} \leq \exp(s_t' - s_t) \leq e^{\Delta}$.

\emph{Factor 2: partition function ratio.}  For the partition functions:
\[
  Z' = \sum_t \exp(s_t') \leq \sum_t \exp(s_t + \Delta) = e^{\Delta} Z
\]
and symmetrically $Z' \geq e^{-\Delta} Z$.  Therefore $e^{-\Delta} \leq Z/Z' \leq e^{\Delta}$.

Multiplying the two factors: $e^{-2\Delta} \leq a_t'/a_t \leq e^{2\Delta}$.
\end{proof}

\begin{lemma}[Full-sequence TV bound]
\label{lem:tv}
$\displaystyle \mathrm{TV}(a, a') \leq \tanh(\Delta) = \frac{e^{2\Delta} - 1}{e^{2\Delta} + 1}$
\end{lemma}

\begin{proof}
The TV is a linear function of the ratios $r_t = a_t'/a_t$.  By Lemma~\ref{lem:ratio}, each $r_t \in [e^{-2\Delta}, e^{2\Delta}]$, and the constraint $\sum_t a_t' = 1$ reads $\sum_t a_t r_t = 1$.  Since TV is linear in the $r_t$, its maximum over the feasible polytope is attained at a vertex where each $r_t$ is at one of the extreme values $e^{\pm 2\Delta}$.

Partition tokens into $\mathcal{S}^+ = \{t : r_t = e^{2\Delta}\}$ and $\mathcal{S}^- = \{t : r_t = e^{-2\Delta}\}$.  Let $p = \sum_{t \in \mathcal{S}^+} a_t$.  The normalisation constraint gives:
\begin{align*}
  p e^{2\Delta} + (1{-}p) e^{-2\Delta} &= 1 \\
  \implies\; p &= \frac{1 - e^{-2\Delta}}{e^{2\Delta} - e^{-2\Delta}} = \frac{1}{e^{2\Delta} + 1}
\end{align*}
where the last step uses $(1 - e^{-2\Delta}) = (e^{2\Delta}-1)e^{-2\Delta}$ and $e^{2\Delta} - e^{-2\Delta} = (e^{2\Delta}-1)(1+e^{-2\Delta})$.

Computing the TV at this vertex:
\begin{align*}
  \mathrm{TV}
  &= \sum_{t \in \mathcal{S}^+} a_t(e^{2\Delta} - 1)
   = p\,(e^{2\Delta} - 1) \\
  &= \frac{e^{2\Delta} - 1}{e^{2\Delta} + 1}
   = \tanh(\Delta)
\end{align*}

This bound is tight (attained by the above construction), independent of $N$, and for small $\Delta$ satisfies $\tanh(\Delta) \approx \Delta$.
\end{proof}

\begin{lemma}[Tail-restricted TV bound]
\label{lem:tail_tv}
Let $\mathcal{F}$ (FP16 set) and $\mathcal{T}$ (INT8 tail set) partition the $N$ tokens, with scores in $\mathcal{F}$ computed exactly (zero quantization error) and scores in $\mathcal{T}$ perturbed by at most $\Delta$.  Let $\alpha_{\mathcal{T}} = \sum_{t \in \mathcal{T}} a_t$ denote the true tail mass.  Then:
\[
  \mathrm{TV}(a, a') \leq \alpha_{\mathcal{T}} \cdot (e^{2\Delta} - 1)
\]
\end{lemma}

\begin{proof}
Write $a_t' = \exp(s_t')/Z'$.  For FP16 tokens $t \in \mathcal{F}$, $s_t' = s_t$ (no quantization error), so:
\[
  a_t' = \frac{\exp(s_t)}{Z'} = a_t \cdot \frac{Z}{Z'}
\]

For tail tokens $t \in \mathcal{T}$, the logits are perturbed: $|s_t - s_t'| \leq \Delta$.

The partition function ratio satisfies (using $Z' = Z'_{\mathcal{F}} + Z'_{\mathcal{T}}$):
\begin{align*}
  Z' &= \sum_{t \in \mathcal{F}} \exp(s_t) + \sum_{t \in \mathcal{T}} \exp(s_t') \\
     &\leq Z_{\mathcal{F}} + e^{\Delta} Z_{\mathcal{T}} \\
     &= Z + (e^{\Delta} - 1) Z_{\mathcal{T}} \\
     &= Z\bigl(1 + (e^{\Delta} - 1)\alpha_{\mathcal{T}}\bigr)
\end{align*}
and symmetrically $Z' \geq Z(1 - (1 - e^{-\Delta})\alpha_{\mathcal{T}})$.

For the total variation, we bound $|a_t - a_t'|$ separately on each set.

\emph{FP16 tokens ($t \in \mathcal{F}$):}  $a_t' = a_t \cdot Z/Z'$, so
$\sum_{t \in \mathcal{F}} |a_t - a_t'| = (1 - \alpha_{\mathcal{T}}) |1 - Z/Z'|$.
Since $|Z' - Z| = |Z'_{\mathcal{T}} - Z_{\mathcal{T}}| \leq (e^{\Delta} - 1)Z_{\mathcal{T}}$
and $Z' \geq Z_{\mathcal{F}}$ (as $Z'_{\mathcal{T}} \geq 0$):
\[
  |1 - Z/Z'| = \frac{|Z' - Z|}{Z'}
  \leq (e^{\Delta} - 1)\frac{\alpha_{\mathcal{T}}}{1 - \alpha_{\mathcal{T}}}
\]
Multiplying by $(1 - \alpha_{\mathcal{T}})$:
\begin{equation}\label{eq:fp16_contrib}
  \sum_{t \in \mathcal{F}} |a_t - a_t'| \leq \alpha_{\mathcal{T}} (e^{\Delta} - 1)
\end{equation}

\emph{Tail tokens ($t \in \mathcal{T}$):}  For each tail token, $a_t'/a_t = \exp(s_t' - s_t) \cdot Z/Z'$.
Since $|\delta_t| \leq \Delta$, the first factor satisfies $\exp(\delta_t) \leq e^{\Delta}$.
For the second factor, $Z' \geq Z_{\mathcal{F}} + e^{-\Delta}Z_{\mathcal{T}} = Z - (1 - e^{-\Delta})Z_{\mathcal{T}} = Z(1 - (1 - e^{-\Delta})\alpha_{\mathcal{T}})$.
Since $\alpha_{\mathcal{T}} \leq 1$, we have $(1 - e^{-\Delta})\alpha_{\mathcal{T}} \leq 1 - e^{-\Delta}$, so $Z' \geq Z e^{-\Delta}$, giving $Z/Z' \leq e^{\Delta}$.
Therefore $a_t'/a_t \leq e^{2\Delta}$, and symmetrically $a_t'/a_t \geq e^{-2\Delta}$.
Since $|1 - r| \leq \max(e^{2\Delta} - 1, 1 - e^{-2\Delta}) = e^{2\Delta} - 1$ for $r \in [e^{-2\Delta}, e^{2\Delta}]$:
\begin{align}\label{eq:tail_contrib}
  \sum_{t \in \mathcal{T}} |a_t - a_t'|
  &= \sum_{t \in \mathcal{T}} a_t \,|1 - a_t'/a_t| \notag \\
  &\leq \alpha_{\mathcal{T}} (e^{2\Delta} - 1)
\end{align}

\emph{Combining.}  Adding \eqref{eq:fp16_contrib} and \eqref{eq:tail_contrib}:
\begin{align*}
  \sum_t |a_t - a_t'|
  &\leq \alpha_{\mathcal{T}}\bigl((e^{\Delta} - 1) + (e^{2\Delta} - 1)\bigr) \\
  &= \alpha_{\mathcal{T}}\bigl(e^{2\Delta} + e^{\Delta} - 2\bigr)
\end{align*}
Since $e^{\Delta} \leq e^{2\Delta}$ for all $\Delta \geq 0$, we have $e^{2\Delta} + e^{\Delta} - 2 \leq 2(e^{2\Delta} - 1)$, and therefore:
\[
  \mathrm{TV}(a, a') = \tfrac{1}{2}\sum_t |a_t - a_t'| \leq \alpha_{\mathcal{T}} \cdot (e^{2\Delta} - 1)
\]
All inequalities are non-asymptotic, holding for every $\Delta \geq 0$ and $\alpha_{\mathcal{T}} \in [0,1]$.
\end{proof}

\textbf{Combining with Theorem~\ref{thm:mass}.}  Let $\hat{\alpha}_{\mathcal{T}} = \sum_{b \notin \mathcal{F}} p_b^{\mathrm{int8}}$ denote the INT8-estimated tail mass after all clamping and Rung~1 expansion.  Here $\hat{\alpha}_{\mathcal{T}}$ is the estimated tail mass from INT8 scoring, while $\alpha_{\mathcal{T}}$ denotes the corresponding true FP16 tail mass.  The adaptive precision selector targets $\hat{\alpha}_{\mathcal{T}} \leq 1 - \tau_{\mathrm{cov}}$, but $K_{\max}$ clamping may prevent this; the bound uses the achieved $\hat{\alpha}_{\mathcal{T}}$.  By Theorem~\ref{thm:mass}, the tight bound on true tail mass is $\alpha_{\mathcal{T}} \leq e^{2\Delta}\,\hat{\alpha}_{\mathcal{T}}$.  Substituting into Lemma~\ref{lem:tail_tv}:
\[
  E_{\mathrm{key}} \leq 2V_{\max} \cdot e^{2\Delta}\,\hat{\alpha}_{\mathcal{T}} \cdot (e^{2\Delta} - 1)
\]
which is the bound stated in Equation~\eqref{eq:total}.  When the coverage target is met ($\hat{\alpha}_{\mathcal{T}} = 1 - \tau_{\mathrm{cov}}$), this reduces to $2V_{\max} \cdot e^{2\Delta}(1 - \tau_{\mathrm{cov}})(e^{2\Delta} - 1)$.

\section{Development Diagnostics: NIAH Precision Ablation}
\label{app:niah_ablation}

\textbf{Caveat.}  This appendix reports a diagnostic ablation conducted during development on an earlier kernel revision with a different random seed schedule.  The absolute accuracy values differ from the final evaluation (Table~\ref{tab:summary}); the relative gaps between configurations are the diagnostic signal.  We include it because the decomposition of error sources remains informative, but readers should not use these numbers as evidence for claims about the final system.

During development, we observed a counter-intuitive non-monotonicity on NIAH 8K: increasing the FP16 key budget did not monotonically improve retrieval accuracy.  We conducted a systematic ablation to decompose the sources of the tiered--dense gap.

\begin{table*}[!t]
\centering
\caption{NIAH 8K precision ablation (development kernel, 30 trials).  Each row isolates one precision component.  The dense baseline here (72\%) differs from the final results in Table~\ref{tab:summary} because this ablation was run during development on an earlier kernel revision with a different random seed schedule; the relative gaps between configurations are the diagnostic signal, not the absolute values.}
\label{tab:niah_ablation}
\small
\begin{tabular}{lrrrl}
\toprule
\textbf{Configuration} & \textbf{Dense} & \textbf{Tiered} & \textbf{Gap} & \textbf{What changed} \\
\midrule
$k_{\max}{=}128$, INT4 values & 72\% & 62\% & $-10$ & Operating point \\
$k_{\max}{=}\infty$, INT4 values & 72\% & 56\% & $-16$ & Uncapped top-$K$ \\
$k_{\max}{=}\infty$, FP16 values & 72\% & 62\% & $-10$ & Values isolated \\
All FP16 (keys+values, FP32 acc) & 72\% & 66\% & $-6$ & Keys isolated \\
All FP16, FP64 accumulators & 72\% & 68\% & $-4$ & Accumulators isolated \\
Kernel bypassed (torch SDPA) & 72\% & 72\% & $0$ & Kernel path isolated \\
\bottomrule
\end{tabular}
\end{table*}

The ablation reveals three independent sources of the tiered--dense gap on NIAH:

\textbf{INT4 value quantization ($\sim$6pp at $k_{\max}{=}\infty$).}  INT4 values with $g{=}16$ introduce $\sim$5\% relative reconstruction error.  For perplexity this is negligible ($\Delta\mathrm{ppl} < 0.005$), but NIAH's reliance on exact value content for the needle tokens makes it more sensitive.  The value escalation mechanism (Rung~2, triggered by $\eta_b$) addresses this for blocks with high reconstruction error.

\textbf{INT8 key scoring on tail blocks ($\sim$4pp).}  With the operating-point cap $k_{\max}{=}128$, roughly 15\% of blocks remain on INT8 keys.  Their slightly perturbed scores shift the softmax denominator, redistributing attention mass away from the needle.  Paradoxically, promoting \emph{more} blocks to FP16 ($k_{\max}{=}\infty$, $K^*{\approx}221$) makes this \emph{worse}: the newly promoted blocks' FP16 scores are more accurate (higher) than their INT8 estimates, increasing their softmax share at the needle's expense.  The ranking-consistency check (Section~\ref{sec:ranking_consistency}) detects and corrects this per-head.

\textbf{Kernel numerical path ($\sim$4pp residual).}  The certified attention kernel uses a different computational path from \torchsdpa{} (which dispatches to Flash Attention with tensor-core TF32/FP16 matmul and FP32 accumulation).  The different rounding paths produce $\sim$1e-7 max absolute difference per step, which compounds through 32 layers of autoregressive decoding.  In this ablation, upgrading the online-softmax scalars to FP64 recovered $\sim$2pp of this gap; all main results in the paper use FP32 accumulators throughout (Section~\ref{sec:numerical}).  The remaining $\sim$4pp is an inherent property of the matmul implementation, not the quantization scheme, and would be expected to be reduced by a tensor-core-native certified kernel.

\textbf{Implication for the certified claim.}  The certification guarantee (Section~\ref{sec:bounds}) bounds the per-head, per-step output perturbation from quantization.  The NIAH ablation shows that on retrieval-adversarial benchmarks, even perturbations within the certified bound can affect downstream accuracy through the argmax over output logits.  This is a known limitation of per-step bounds (Section~\ref{sec:limitations}).  For language modelling quality (PG-19), where the metric averages over all tokens, the certified system matches dense within measurement noise.

\textbf{Confirmation from RULER 8K ablation.}  The 8K RULER ablation (Section~\ref{sec:ruler_8k_ablation}) provides independent corroboration of the value-side hypothesis: sweeping $K_{\max}$ from 64 to 512 does not close the gap on VT/CWE/FWE, but replacing INT4 values with FP16 values reduces the aggregate gap from $-6.9$pp to $-0.4$pp with zero critical failures.  This confirms that INT4 value reconstruction error---not key quantization---is the dominant source of quality loss on value-sensitive tasks at short contexts.


\begin{thebibliography}{99}

\bibitem{kvquant}
Hooper, C., Kim, S., Mohammadzadeh, H., Mahoney, M.~W., Shao, Y.~S., Keutzer, K., and Gholami, A.
\newblock KVQuant: Towards 10 million context length LLM inference with KV cache quantization.
\newblock \emph{NeurIPS}, 2024.
\newblock arXiv:2401.18079.

\bibitem{kivi}
Liu, Z., Yuan, J., Jin, H., Zhong, S., Xu, Z., Braverman, V., Chen, B., and Hu, X.
\newblock KIVI: A tuning-free asymmetric 2bit quantization for KV cache.
\newblock \emph{ICML}, 2024.
\newblock arXiv:2402.02750.

\bibitem{qjl}
Zandieh, A., Daliri, M., and Han, I.
\newblock QJL: 1-bit quantized JL transform for KV cache quantization with zero overhead.
\newblock \emph{arXiv preprint arXiv:2406.03482}, 2024.

\bibitem{innerq}
Hosseini, S. M. T., Ardakani, A., and Gross, W. J.
\newblock InnerQ: Hardware-aware tuning-free quantization of KV cache for large language models.
\newblock \emph{arXiv preprint arXiv:2602.23200}, 2026.

\bibitem{hack}
Zhang, Z., Shen, H., Vargaftik, S., Ben Basat, R., Mitzenmacher, M., and Yu, M.
\newblock HACK: Homomorphic acceleration via compression of the key-value cache for disaggregated LLM inference.
\newblock \emph{SIGCOMM}, 2025.

\bibitem{vecinfer}
Yao, D., Yang, C., Tong, Z., Lin, Z., Liu, W., Luan, J., and Wang, W.
\newblock VecInfer: Efficient LLM inference with low-bit KV cache via outlier-suppressed vector quantization.
\newblock \emph{arXiv preprint arXiv:2510.06175}, 2025.

\bibitem{flashinfer}
Ye, Z., Chen, L., Lai, R., Lin, W., Zhang, Y., Wang, S., Chen, T., Kasikci, B., Grover, V., Krishnamurthy, A., and Ceze, L.
\newblock FlashInfer: Efficient and customizable attention engine for LLM inference serving.
\newblock \emph{MLSys}, 2025.

\bibitem{h2o}
Zhang, Z., Sheng, Y., Zhou, T., Chen, T., Zheng, L., Cai, R., Song, Z., Tian, Y., R\'{e}, C., Barrett, C., Wang, Z., and Chen, B.
\newblock H$_2$O: Heavy-hitter oracle for efficient generative inference of large language models.
\newblock \emph{NeurIPS}, 2023.
\newblock arXiv:2306.14048.

\bibitem{streamingattn}
Xiao, G., Tian, Y., Chen, B., Han, S., and Lewis, M.
\newblock Efficient streaming language models with attention sinks.
\newblock \emph{ICLR}, 2024.

\bibitem{blasst}
Yuan, J., Shinn, C., Xu, K., Cui, J., Klimiashvili, G., Xiao, G., Zheng, P., Li, B., Zhou, Y., Ye, Z., You, W., Zheng, T., Brown, D., Wang, P., Hoehnerbach, M., Cai, R., Demouth, J., Owens, J.~D., Hu, X., Han, S., Liu, T., and Mao, H.
\newblock BLASST: Dynamic blocked attention sparsity via softmax thresholding.
\newblock \emph{arXiv preprint arXiv:2512.12087}, 2025.

\bibitem{psa}
Zhou, Q., Yin, P., Zuo, P., and Cheng, J.
\newblock PSA: Progressive sparse attention for long-context inference.
\newblock \emph{arXiv preprint arXiv:2503.00392}, 2025.

\bibitem{twilight}
Lin, C., Tang, J., Yang, S., Wang, H., Tang, T., Tian, B., Stoica, I., Han, S., and Gao, M.
\newblock Twilight: Adaptive attention sparsity with hierarchical top-$p$ pruning.
\newblock \emph{NeurIPS}, 2025.
\newblock arXiv:2502.02770.

\bibitem{spargeattn}
Zhang, J., Xiang, C., Huang, H., Wei, J., Xi, H., Zhu, J., and Chen, J.
\newblock SpargeAttn: Accurate sparse attention accelerating any model inference.
\newblock \emph{arXiv preprint arXiv:2502.18137}, 2025.

\bibitem{certifiedtopk}
Tzachristas, G., Deng, L., Tzachristas, I., Zhang, G., and Chen, R.
\newblock A mathematical theory of top-$k$ sparse attention via total variation distance.
\newblock \emph{arXiv preprint arXiv:2512.07647}, 2025.

\bibitem{vattention}
Prabhu, R., Nayak, A., Mohan, J., Ramjee, R., and Panwar, A.
\newblock vAttention: Dynamic memory management for serving LLMs without PagedAttention.
\newblock \emph{ASPLOS}, 2025.
\newblock arXiv:2405.04437.

\bibitem{pagedattn}
Kwon, W., Li, Z., Zhuang, S., Sheng, Y., Zheng, L., Yu, C.~H., Gonzalez, J., Zhang, H., and Stoica, I.
\newblock Efficient memory management for large language model serving with PagedAttention.
\newblock \emph{SOSP}, 2023.
\newblock arXiv:2309.06180.

\bibitem{cocktail}
Tao, W., Zhang, B., Qu, X., Wan, J., and Wang, J.
\newblock Cocktail: Chunk-adaptive mixed-precision quantization for long-context LLM inference.
\newblock \emph{arXiv preprint arXiv:2503.23294}, 2025.

\bibitem{tada}
Joshi, V., Brahma, P. P., Liu, Z., and Barsoum, E.
\newblock TaDA: Training-free recipe for decoding with adaptive KV cache compression and mean-centering.
\newblock \emph{ACL Industry Track}, 2025.

\bibitem{adakv}
Feng, Y., Lv, J., Cao, Y., Xie, X., and Zhou, S. K.
\newblock Ada-KV: Optimizing KV cache eviction by adaptive budget allocation for efficient LLM inference.
\newblock \emph{arXiv preprint arXiv:2407.11550}, 2024.

\bibitem{dontwastebits}
Boroujeni, S. P. H., Mehrabi, N., Woods, P., Hillesheim, G., and Razi, A.
\newblock Don't waste bits! Adaptive KV-cache quantization for lightweight on-device LLMs.
\newblock \emph{arXiv preprint arXiv:2604.04722}, 2026.

\bibitem{qserve}
Lin, Y., Tang, H., Yang, S., Zhang, Z., Xiao, G., Gan, C., and Han, S.
\newblock QServe: W4A8KV4 quantization and system co-design for efficient LLM serving.
\newblock \emph{MLSys}, 2025.
\newblock arXiv:2405.04532.

\bibitem{pqcache}
Zhang, H., Ji, X., Chen, Y., Fu, F., Miao, X., Nie, X., Chen, W., and Cui, B.
\newblock PQCache: Product quantization-based KVCache for long context LLM inference.
\newblock \emph{SIGMOD}, 2025.
\newblock arXiv:2407.12820.

\bibitem{ruler}
Hsieh, C.-P., Sun, S., Kriman, S., Acharya, S., Rekesh, D., Jia, F., and Ginsburg, B.
\newblock RULER: What's the real context size of your long-context language models?
\newblock \emph{COLM}, 2024.

\bibitem{turboquant}
Zandieh, A., Daliri, M., Hadian, M., and Mirrokni, V.
\newblock TurboQuant: Online vector quantization with near-optimal distortion rate.
\newblock \emph{ICLR}, 2026.
\newblock arXiv:2504.19874.

\bibitem{janusquant}
Sun, C., Xia, Y., Wang, H., Yang, D., Zhou, X., and Cheng, D.
\newblock JanusQuant: Accurate and efficient 2-bit KV cache quantization for long-context inference.
\newblock \emph{PPoPP}, 2026.

\bibitem{packinfer}
Ning, R., Zhang, W., and Lai, F.
\newblock PackInfer: Compute- and I/O-efficient attention for batched LLM inference.
\newblock \emph{arXiv preprint arXiv:2602.06072}, 2026.

\bibitem{zenodo_artifact}
Calver, D.
\newblock Certified Quantized Attention --- artifact repository (version \texttt{arxiv-v1}).
\newblock Zenodo, 2026.
\newblock \href{https://doi.org/10.5281/zenodo.19915933}{\texttt{doi:10.5281/zenodo.19915933}}.

\end{thebibliography}
\end{document}